\documentclass[letterpaper]{article}
\usepackage[utf8]{inputenc}
\usepackage{amsmath}
\usepackage{amsthm}
\usepackage{amssymb}
\usepackage{graphicx}
\usepackage[margin=1in,letterpaper]{geometry}
\usepackage{enumitem}
\usepackage{xcolor}
\usepackage[colorlinks=true,allcolors=blue]{hyperref}
\usepackage{subcaption}
\usepackage{algorithm}
\usepackage[authoryear]{natbib}
\newtheorem{theorem}{Theorem}

\newtheorem{lemma}{Lemma}

\DeclareMathOperator*{\argmin}{arg\,min}

\begin{document}

\title{Biclustering with Alternating K-Means}
\author{Nicolas Fraiman \\  \href{mailto:fraiman@email.unc.edu}{fraiman@email.unc.edu} 
\and Zichao Li \\ 
\href{mailto:lizichao@live.unc.edu}{lizichao@live.unc.edu}}
\date{}
\maketitle

\begin{abstract}
    Biclustering is the task of simultaneously clustering the rows and columns of the data matrix into different subgroups such that the rows and columns within a subgroup exhibit similar patterns. In this paper, we consider the case of producing block-diagonal biclusters. We provide a new formulation of the biclustering problem based on the idea of minimizing the empirical clustering risk. We develop and prove a consistency result with respect to the empirical clustering risk. Since the optimization problem is combinatorial in nature, finding the global minimum is computationally intractable. In light of this fact, we propose a simple and novel algorithm that finds a local minimum by alternating the use of an adapted version of the $k$-means clustering algorithm between columns and rows. We evaluate and compare the performance of our algorithm to other related biclustering methods on both simulated data and real-world gene expression data sets. The results demonstrate that our algorithm is able to detect meaningful structures in the data and outperform other competing biclustering methods in various settings and situations.
\end{abstract}

\section{Introduction}\label{sec:intro}
In many fields of application, the data can be represented by a matrix, and people are interested in the task of simultaneously clustering the rows and columns of the data matrix into different subgroups such that the rows and columns within a subgroup exhibit similar patterns. This general task has been studied in many different application domains. For example, in gene expression analysis, people seek to identify subgroups of genes that have similar expression levels within corresponding subgroups of conditions \citep{cheng2000biclustering}. In text mining, people attempt to recognize subgroups of documents that have similar properties with respect to corresponding subgroups of words \citep{dhillon2001co}. In collaborative filtering, people wish to detect subgroups of customers with similar preferences toward corresponding subgroups of products \citep{hofmann1999latent}. The most common name of the task is biclustering \citep{cheng2000biclustering, tanay2002discovering, kluger2003spectral, prelic2006systematic, mankad2014biclustering}, although it is also known by other names such as co-clustering \citep{dhillon2001co, dhillon2003information, cho2004minimum, banerjee2007generalized, chi2020provable}, block clustering \citep{govaert2003clustering, govaert2005algorithm, govaert2008block}, and direct clustering \citep{hartigan1972direct}.

Over the years, a large number of biclustering methods have been proposed. Comprehensive reviews of different biclustering methods can be found in \citet{madeira2004biclustering} and \citet{tanay2005biclustering}. The biclustering methods could be classified into different groups based on the structure of the produced biclusters. Figure \ref{bcstructure} shows three types of bicluster structures that could be obtained after appropriate row and column reordering:
\begin{enumerate}
    \item In Figure \ref{overlapping}, the biclusters are arbitrarily positioned and can overlap with each other. The majority of the biclustering methods produce this type of biclusters, including \citet{cheng2000biclustering}, CTWC \citep{getz2000coupled}, ISA \citep{bergmann2003iterative}, SAMBA \citep{tanay2004revealing}, plaid models \citep{lazzeroni2002plaid}, OPSM \citep{ben2002discovering}, xMOTIFs \citep{murali2002extracting}, and others \citep{li2009qubic, shabalin2009finding, hochreiter2010fabia}.
    \item In Figure \ref{checkerboard}, the biclusters are non-overlapping and follow a checkerboard structure. The biclustering methods that produce this type of biclusters include spectral biclustering \citep{kluger2003spectral}, SSVD \citep{lee2010biclustering}, sparse biclustering \citep{tan2014sparse}, convex biclustering \citep{chi2017convex}, profile likelihood biclustering \citep{flynn2020profile}, high-order spectral clustering \citep{han2020exact}, and others \citep{dhillon2003information, govaert2003clustering, cho2004minimum, chen2013biclustering}.
    \item In Figure \ref{exclusive}, the biclusters are rectangular diagonal blocks in the data matrix. In this case, there exist $k$ mutually exclusive and exhaustive clusters of rows, and $k$ corresponding mutually exclusive and exhaustive clusters of columns. Our method, named alternating $k$-means biclustering, produces this type of biclusters.
\end{enumerate}
Some biclustering methods \citep{segal2001rich, tang2001interrelated, wang2002clustering, gu2008bayesian, sill2011robust} produce other types of bicluster structures. A more detailed discussion is provided in \citet{madeira2004biclustering}.

\begin{figure}[t]
\centering
\begin{subfigure}{.3\textwidth}
\centering
\includegraphics[width=0.8\linewidth]{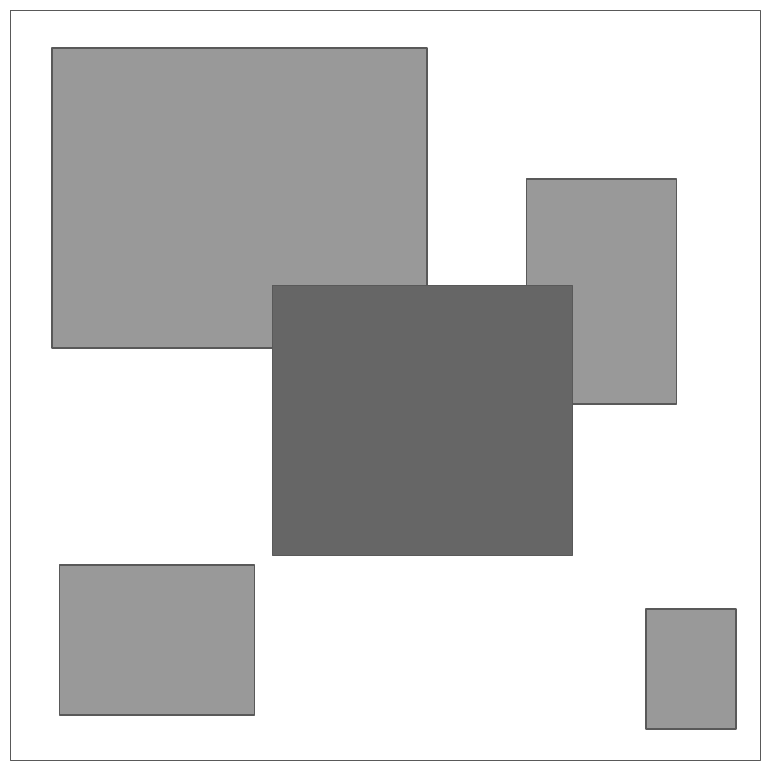}
\caption{}
\label{overlapping}
\end{subfigure}
\begin{subfigure}{.3\textwidth}
\centering
\includegraphics[width=0.8\linewidth]{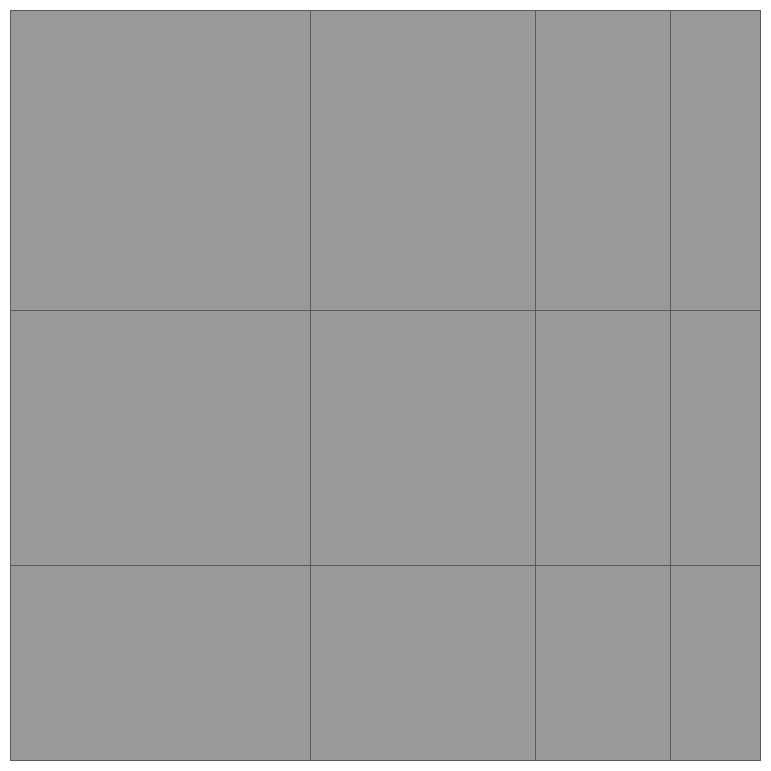}
\caption{}
\label{checkerboard}
\end{subfigure}
\begin{subfigure}{.3\textwidth}
\centering
\includegraphics[width=0.8\linewidth]{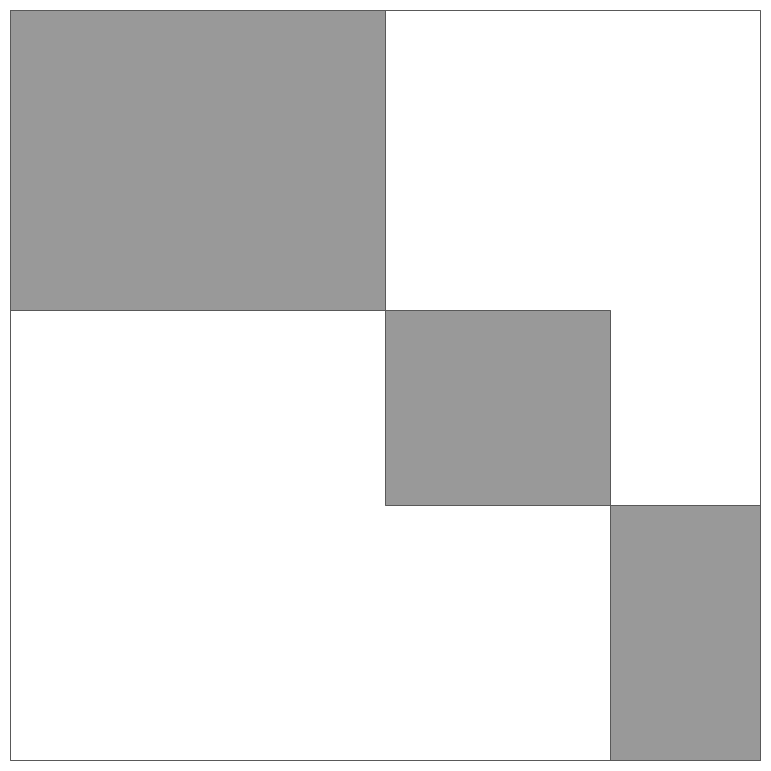}
\caption{}
\label{exclusive}
\end{subfigure}
\caption{Different types of bicluster structures (after row and column reordering). (a) Arbitrarily positioned overlapping biclusters, (b) non-overlapping biclusters with checkerboard structure, and (c) block-diagonal biclusters.}
\label{bcstructure}
\end{figure}

In general, methods that produce overlapping biclusters are more complex and flexible, whereas methods that produce non-overlapping biclusters are easier to interpret and visualize. Our method produces block-diagonal biclusters, which arise naturally in several applications. For example, in gene expression analysis, block-diagonal biclusters help people divide genes into different groups associated with different types of cancer tissues. In text mining, block-diagonal biclusters help people group documents into distinct clusters based on distinct sets of terms. One particular example is shown in Figure \ref{fig:example_heatmap}, where we apply our biclustering method to the breast and colon cancer gene expression data set \citep{de2008clustering} consisting of 104 samples and 182 genes. As we can see, our biclustering method partitions samples into two groups that almost perfectly correspond to the two cancer types, and it also identifies two groups of genes associated with breast and colon cancer.

\begin{figure}[ht]
\centering
\includegraphics[width=0.5\textwidth, height=0.5\textwidth]{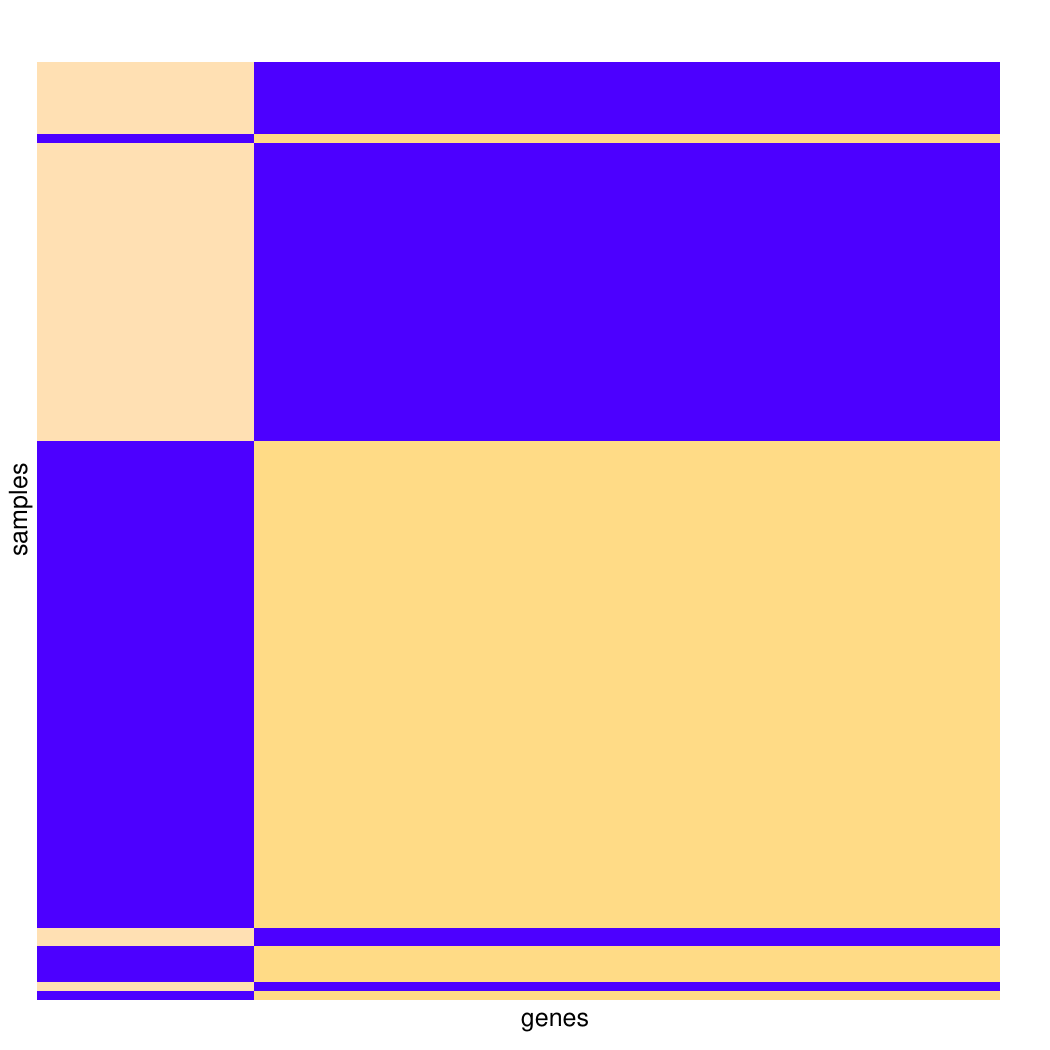}
\caption{Heatmap of the estimated mean matrix from our biclustering method on the breast and colon cancer data set. The rows (samples) are ordered by true cancer type: first 42 are colon cancer tissues, and the other 62 are breast cancer tissues. The columns (genes) are reordered based on the produced biclusters for visualization purposes.}
\label{fig:example_heatmap}
\end{figure}

The key difference between clustering and biclustering is that clustering is based on \emph{global} patterns whereas biclustering is based on \emph{local} patterns. More specifically, when performing clustering on the rows of the data matrix, all the columns are taken into consideration. In contrast, when performing biclustering on the data matrix, the rows are clustered into different groups based on different subsets of columns. This characteristic of biclustering inspired us to develop a ``local'' version of $k$-means clustering: instead of performing $k$-means clustering on the rows using all the columns, we could consider only using different subsets of columns. In other words, the cluster centers could be defined using different subsets of columns instead of all the columns. Starting from this simple idea, we adapt the formulation and algorithm of $k$-means clustering to a biclustering method by making several important modifications, and the details of our formulation and algorithm are given in Section \ref{sec:theory} and Section \ref{sec:algorithm}, respectively. Notably, our alternating $k$-means biclustering algorithm is conceptually as simple as $k$-means clustering, and it has the extra advantage of being able to discover local patterns as opposed to global patterns. These two characteristics make our algorithm an ideal candidate to serve as a baseline for biclustering problems even when our bicluster structure might not be flexible enough.

Our main contributions can be summarized as follows:
\begin{enumerate}
    \item We provide a new formulation of the biclustering problem based on the idea of minimizing the empirical clustering risk. The formulation is adapted from the $k$-means clustering problem, with two important changes with respect to the definitions of cluster centers and norm. We further develop and prove a consistency result with respect to the empirical clustering risk, which is generally quite rare for biclustering methods.
    \item Since minimizing the empirical clustering risk is a combinatorial optimization problem, finding the global minimum is computationally intractable. In light of this fact, we propose a simple and novel algorithm that finds a local minimum by alternately applying an adapted version of the $k$-means clustering algorithm between columns and rows. The simplicity of our algorithm makes it easy to understand, implement, and interpret. The R package \texttt{akmbiclust}, available on CRAN, implements our alternating $k$-means biclustering algorithm.
    \item We empirically evaluate and compare the performance of our method to other related biclustering methods on both simulated data and real-world gene expression data sets. The empirical results have demonstrated that our method is able to detect meaningful structures in the data and outperform other competing biclustering methods in various settings and situations.
\end{enumerate}

The rest of this paper is organized as follows. In Section \ref{sec:theory}, we formulate the task of biclustering as an optimization problem and present a consistency result. In Section \ref{sec:proof}, we provide a rigorous proof of the consistency result. In Section \ref{sec:probabilistic}, we describe a probabilistic interpretation of the optimization problem. In Section \ref{sec:algorithm}, we present a simple algorithm that finds the local optimum by alternating the use of $k$-means clustering between columns and rows. In Section \ref{sec:penalization}, we propose extending our method by adding penalization terms. In Section \ref{sec:simulation}, we evaluate and compare the algorithm's performance on simulated data to other related biclustering algorithms. In Section \ref{sec:applications}, we apply the algorithm to three cancer gene expression data sets, demonstrating its advantage over other related biclustering algorithms in terms of sample misclassification rate. In Section \ref{sec:discussion}, we conclude with a discussion.

\section{Problem Formulation and Consistency Result}\label{sec:theory}
Suppose we have a $n \times m$ matrix $\mathbf{X}$ representing $n$ data points $X_1, \dots, X_n \in \mathbb{R}^m$. A typical example is the gene expression matrix, with rows corresponding to genes and columns corresponding to conditions. We formulate the task of biclustering on $\mathbf{X}$ as partitioning the $n$ rows and $m$ columns into $k$ groups to get $k$ biclusters, as shown in Figure \ref{exclusive}. More specifically, let $J = \{1, \dots, n\}$ be the set of row indices, then $J$ could be partitioned into $k$ disjoint nonempty sets $J_1, \dots, J_k$, where $J_1 \cup \dots \cup J_k = J$. Similarly, let $I = \{1, \dots, m\}$ be the set of column indices, then $I$ could also be partitioned into $k$ disjoint nonempty sets $I_1, \dots, I_k$, where $I_1 \cup \dots \cup I_k = I$. The $k$ groups of row indices $J_1, \dots, J_k$ and column indices $I_1, \dots, I_k$ could be viewed as $k$ biclusters. Note that under this definition of biclustering every row and column in the matrix $\mathbf{X}$ belongs to one and only one bicluster. In other words, the rows and columns in the biclusters are exhaustive and exclusive.

For simplicity of notation, we assume a vector is a row vector unless otherwise stated. For any $X = (x_1, \dots, x_m) \in \mathbb{R}^m$, let $X(I_j) = (x_i)_{i \in I_j}$. For example, let $X = (1, 3, 4, 7)$, and $I_1 = \{1, 3\}, I_2 = \{2, 4\}$. Then $X(I_1) = (1, 4), X(I_2) = (3, 7)$. The space of $X(I_j)$ is defined as $\mathbb{R}^{I_j}$. We define a special norm on $\mathbb{R}^{I_j}$, called dimensionality-normalized norm. For any $X \in \mathbb{R}^{I_j}$, let $l_j = |I_j|$ denote the cardinality of the index set $I_j$, then it is also the dimension of the space $\mathbb{R}^{I_j}$, and the dimensionality-normalized norm of $X$ is defined as 
\begin{equation*}
    ||X||_{dn} = \sqrt{\frac{\sum_{i \in I_j} x_i^2}{l_j}}.
\end{equation*}
The name ``dimensionality-normalized norm'' comes from the following simple relationship between the dimensionality-normalized norm and the Euclidean norm:
\begin{equation*}
    ||X||_{dn}^2 = \frac{||X||_2^2}{l_j}.
\end{equation*}
Our method seeks to find the $k$ groups of column indices $I_j, 1 \le j \le k$ and the $k$ cluster centers $c_j \in \mathbb{R}^{I_j}, 1 \le j \le k$ such that the following objective function is minimized:
\begin{equation}\label{def:objective}
    \sum_{i=1}^{n} \min_{1 \le j \le k} ||X_i(I_j) - c_j||_{dn}^2.
\end{equation}
The corresponding $k$ groups of row indices $J_t, 1 \le t \le k$ can be obtained by selecting all the rows that are ``closest'' to cluster center $c_t$:
\begin{equation*}
    J_t = \{ i: \argmin_{1 \le j \le k} ||X_i(I_j) - c_j||_{dn}^2 = t \}, 1 \le t \le k.
\end{equation*}
Note that here ``closest'' is measured by the distance function induced by the dimensionality-normalized norm:
\begin{equation*}
    \text{dist}(X_i(I_j), c_j) = ||X_i(I_j) - c_j||_{dn}, 1 \le j \le k, 1 \le i \le n.
\end{equation*}

Our biclustering method can be viewed as a more complicated version of the traditional $k$-means clustering, which only seeks to find the $k$ cluster centers $c_j \in \mathbb{R}^m, 1 \le j \le k$ such that the following objective function is minimized:
\begin{equation*}
    \sum_{i=1}^{n} \min_{1 \le j \le k} ||X_i - c_j||_2^2.
\end{equation*}
However, it is important to note that there are two key differences:
\begin{enumerate}
    \item The $k$ cluster centers $c_1, \dots, c_k$ in our objective function are not vectors in $\mathbb{R}^m$. Instead, $c_j \in \mathbb{R}^{I_j}$ for $1 \le j \le k$, and the $k$ groups of column indices $I_1, \dots, I_k$ also are parameters that we need to optimize over. In fact, finding the best column partition $I_1, \dots, I_k$ is combinatorial in nature, which makes the optimization problem computationally intractable.
    \item The norm in our objective function is not the Euclidean norm. Instead, it is the dimensionality-normalized norm.
\end{enumerate}

The motivation of using the dimensionality-normalized norm instead of the Euclidean norm in the objective function (\ref{def:objective}) involves both theoretical and empirical aspects:
\begin{enumerate}
    \item Theoretically, we provide a probabilistic interpretation of the objective function, in which dimensionality-normalized norm is the result of normalizing the log-likelihood to ensure fair comparison between vectors of different dimensionality. More details are given in Section \ref{sec:probabilistic}.
    \item Empirically, we observe that using Euclidean norm would often result in biclusters that are of undesirable shape. In the extreme case, when one uses Euclidean norm to produce two biclusters on a $n \times n$ matrix, the result would be one $1 \times (n-1)$ bicluster and one $(n-1) \times 1$ bicluster. This is because when using the Euclidean norm version of the objective function, extremely tall and thin bicluster combined with short and wide bicluster means many terms with a few columns and a few terms with many columns, and the result is a very small value of the objective function. In contrast, using dimensionality-normalized norm would produce biclusters of normal shape, and empirically it also performs much better than using Euclidean norm.
\end{enumerate}

Suppose the data is a sequence of independent random observations $X_1, \dots, X_n \in \mathbb{R}^m$ with the same distribution as a generic random variable $X$ with distribution $\mu$. We minimize the empirical clustering risk
\begin{equation}\label{def:empirical_risk}
    W(\mathbf{I}, \mathbf{c}, \mu_n) = \frac{1}{n} \sum_{i=1}^{n} \min_{1 \le j \le k} ||X_i(I_j) - c_j||_{dn}^2
\end{equation}
over all possible choices of column partitions $\mathbf{I} = \{ I_j\}_{1 \le j \le k}$ and cluster centers $\mathbf{c} = \{ c_j\}_{1 \le j \le k}$. Here, $\mu_n$ is the empirical distribution of the data.

The performance of a clustering scheme given by the column partition $\mathbf{I}$ and cluster centers $\mathbf{c}$ is measured by the clustering risk
\begin{equation}\label{def:cluster_risk}
    W(\mathbf{I}, \mathbf{c}, \mu) = \int \min_{1 \le j \le k} ||x(I_j) - c_j||_{dn}^2 d\mu(x).
\end{equation}
The optimal clustering risk is defined as
\begin{equation}\label{def:optimal_risk}
    W^*(\mu) = \inf_\mathbf{I} \inf_\mathbf{c} W(\mathbf{I}, \mathbf{c}, \mu).
\end{equation}

Let $\delta_n \ge 0$. A column partition $\mathbf{I}_n$ and cluster centers $\mathbf{c}_n$ as a whole is a $\delta_n$-minimizer of the empirical clustering risk if 
\begin{equation*}
    W(\mathbf{I}_n, \mathbf{c}_n, \mu_n) \le W^*(\mu_n) + \delta_n,
\end{equation*}
where $W^*(\mu_n) = \inf_\mathbf{I} \inf_\mathbf{c} W(\mathbf{I}, \mathbf{c}, \mu_n)$. When $\delta_n = 0$, $\mathbf{I}_n$ and $\mathbf{c}_n$ as a whole is called an empirical risk minimizer. Since $\mu_n$ is supported on at most $n$ points, the existence of an empirical risk minimizer is guaranteed. 

The key theoretical result of this paper is the following consistency theorem, which states that the clustering risk of a $\delta_n$-minimizer of the empirical clustering risk converges to the optimal risk as long as $\lim_{n \to \infty} \delta_n = 0$.

\begin{theorem}\label{thm:consistency}
Assume that $\mathbb{E}||X||_2^2 < \infty$. Let $\mathbf{I}_n$ and $\mathbf{c}_n$ be a $\delta_n$-minimizer of the empirical clustering risk. If $\lim_{n \to \infty} \delta_n = 0$, then
\begin{enumerate}
    \item $\lim_{n \to \infty} W(\mathbf{I}_n, \mathbf{c}_n, \mu) = W^*(\mu)$ a.s., and
    \item $\lim_{n \to \infty} \mathbb{E}W(\mathbf{I}_n, \mathbf{c}_n, \mu) = W^*(\mu)$.
\end{enumerate}
\end{theorem}

It is important to point out that we assume a minimizer of the empirical clustering risk can be found. However, finding the global minimum of the empirical clustering risk is a computationally intractable problem due to its combinatorial nature. In light of this fact, we present a simple algorithm in Section \ref{sec:algorithm} that finds a local minimum based on the idea of alternating the use of $k$-means clustering between columns and rows.

\section{Proof of the Consistency Result}\label{sec:proof}
Theorem \ref{thm:consistency} is quite similar to Proposition 2.1 in \citet{biau2008performance}, which showed that for $k$-means clustering, the clustering risk of a $\delta_n$-minimizer of the empirical clustering risk converges to the optimal risk. More specifically, they used 
\begin{align*}
    W(\mathbf{c}, \mu_n) &= \frac{1}{n} \sum_{i=1}^{n} \min_{1 \le j \le k} ||X_i - c_j||_2^2, \\
    W(\mathbf{c}, \mu) &= \int \min_{1 \le j \le k} ||x - c_j||_2^2 d\mu(x), \\
    W^*(\mu) &= \inf_\mathbf{c} W(\mathbf{c}, \mu).
\end{align*}
in place of our $W(\mathbf{I}, \mathbf{c}, \mu_n)$, $W(\mathbf{I}, \mathbf{c}, \mu)$, and $W^*(\mu)$ defined in Equation (\ref{def:empirical_risk}), (\ref{def:cluster_risk}), and (\ref{def:optimal_risk}). Not surprisingly, the steps used to proved Proposition 2.1 in \citet{biau2008performance} could also be adapted to prove Theorem \ref{thm:consistency}, after making some minor changes and proving similar lemmas.

First, we introduce the definition of $L_2$ Wasserstein distance, which is essential for the proof. The $L_2$ Wasserstein distance between two probability measures $\mu_1$ and $\mu_2$ on $\mathbb{R}^m$, with finite second moment, is defined as
\begin{equation*}
    \gamma(\mu_1, \mu_2) = \inf_{X \sim \mu_1, Y \sim \mu_2} (\mathbb{E}||X-Y||_2^2)^{1/2},
\end{equation*}
where the infimum is taken over all joint distributions of two random variables $X$ and $Y$ such that $X$ has distribution $\mu_1$ and $Y$ has distribution $\mu_2$. Our main work is to prove the following lemma, which is similar to inequality (4) in \citet{biau2008performance}.

\begin{lemma}\label{lem:costbound}
For any column partition $\mathbf{I}$ and cluster centers $\mathbf{c}$, 
\begin{equation*}
    \left| W(\mathbf{I}, \mathbf{c}, \mu_1)^{1/2} - W(\mathbf{I}, \mathbf{c}, \mu_2)^{1/2} \right| \le \gamma(\mu_1, \mu_2).
\end{equation*}
\end{lemma}

\begin{proof}
The proof is based on the proof of Lemma 3 in \citet{linder2002learning}. Let $X \sim \mu_1$ and $Y \sim \mu_2$ achieve the infimum defining $\gamma(\mu_1, \mu_2)$. Then
\begin{align*}
    W(\mathbf{I}, \mathbf{c}, \mu_1)^{1/2} &= \left[ \int \min_{1 \le j \le k} ||x(I_j) - c_j||_{dn}^2 d\mu_1(x) \right]^{1/2} \\
    &= \left[ \mathbb{E} \min_{1 \le j \le k} ||X(I_j) - c_j||_{dn}^2 \right]^{1/2} \\
    &= \left[ \mathbb{E} \min_{1 \le j \le k} \frac{||X(I_j) - c_j||_2^2}{l_j} \right]^{1/2} \\
    &\le \left[ \mathbb{E} \min_{1 \le j \le k} \frac{(||X(I_j) - Y(I_j)||_2 + ||Y(I_j) - c_j||_2)^2}{l_j} \right]^{1/2}.
\end{align*}
Using Cauchy–Schwarz inequality, we have
\begin{align*}
&\mathbb{E} \left[ \frac{(||X(I_j) - Y(I_j)||_2 + ||Y(I_j) - c_j||_2)^2}{l_j} \right] \\ 
&= \mathbb{E} \left[ \frac{||X(I_j) - Y(I_j)||_2^2}{l_j} \right] + \mathbb{E} \left[ \frac{||Y(I_j) - c_j||_2^2}{l_j} \right] + 2 \mathbb{E} \left[ \frac{||X(I_j) - Y(I_j)||_2 ||Y(I_j) - c_j||_2}{l_j} \right] \\
&\le \mathbb{E} ||X-Y||_2^2 + \mathbb{E} \left[ \frac{||Y(I_j) - c_j||_2^2}{l_j} \right] + 2 \mathbb{E} \left[ ||X - Y||_2 \cdot \frac{||Y(I_j) - c_j||_2}{\sqrt{l_j}} \right] \\
&\le \mathbb{E} ||X-Y||_2^2 + \mathbb{E} \left[ \frac{||Y(I_j) - c_j||_2^2}{l_j} \right] + 2 \left[ \mathbb{E} ||X - Y||_2^2 \right]^{1/2} \left[ \mathbb{E} \frac{||Y(I_j) - c_j||_2^2}{l_j} \right]^{1/2} \\
&= \left( \left[ \mathbb{E} ||X - Y||_2^2 \right]^{1/2} + \left[ \mathbb{E} \frac{||Y(I_j) - c_j||_2^2}{l_j} \right]^{1/2} \right)^2.
\end{align*}
Consequently
\begin{align*}
    W(\mathbf{I}, \mathbf{c}, \mu_1)^{1/2} &\le \left[ \mathbb{E} \min_{1 \le j \le k} \frac{(||X(I_j) - Y(I_j)||_2 + ||Y(I_j) - c_j||_2)^2}{l_j} \right]^{1/2} \\
    &\le \left[ \mathbb{E} ||X - Y||_2^2 \right]^{1/2} + \left[ \mathbb{E} \min_{1 \le j \le k} \frac{||Y(I_j) - c_j||_2^2}{l_j} \right]^{1/2} \\
    &= \left[ \mathbb{E} ||X - Y||_2^2 \right]^{1/2} + \left[ \mathbb{E} \min_{1 \le j \le k} ||Y(I_j) - c_j||_{dn}^2 \right]^{1/2} \\
    &= \gamma(\mu_1, \mu_2) + W(\mathbf{I}, \mathbf{c}, \mu_2)^{1/2},
\end{align*}
which implies that $W(\mathbf{I}, \mathbf{c}, \mu_1)^{1/2} - W(\mathbf{I}, \mathbf{c}, \mu_2)^{1/2} \le \gamma(\mu_1, \mu_2)$. The other direction can be proved similarly. 
\end{proof}

Having proved Lemma \ref{lem:costbound}, we are ready to prove the following lemma, which is similar to Lemma 4.1 in \citet{biau2008performance}.

\begin{lemma}\label{lem:minimizer}
Let $\mathbf{I}_n$ and $\mathbf{c}_n$ be a $\delta_n$-minimizer of the empirical clustering risk. Then
\begin{equation*}
    W(\mathbf{I}_n, \mathbf{c}_n, \mu)^{1/2} - [\inf_\mathbf{I} \inf_\mathbf{c} W(\mathbf{I}, \mathbf{c}, \mu)]^{1/2} \le 2\gamma(\mu, \mu_n) + \sqrt{\delta_n}.
\end{equation*}
\end{lemma}

\begin{proof}
After replacing $W(\mathbf{c}_n, \mu)$, $\inf_\mathbf{c} W(\mathbf{c}, \mu)$ and some other intermediate variables with $W(\mathbf{I}_n, \mathbf{c}_n, \mu)$, $\inf_\mathbf{I} \inf_\mathbf{c} W(\mathbf{I}, \mathbf{c}, \mu)$ and other corresponding intermediate variables, the same steps that are used to prove Lemma 4.1 in \citet{biau2008performance} could be applied to prove Lemma \ref{lem:minimizer}. Lemma \ref{lem:costbound} is used in the last inequality of the proof.
\end{proof}

Theorem \ref{thm:consistency} follows naturally from Lemma \ref{lem:minimizer} and Lemma 4.2 in \citet{biau2008performance}, which states that $\displaystyle\lim_{n \to \infty} \gamma(\mu, \mu_n) = 0$ a.s., and $\displaystyle\lim_{n \to \infty} \mathbb{E} \gamma^2(\mu, \mu_n) = 0$.

\section{A Probabilistic Interpretation of the Optimization Problem}\label{sec:probabilistic}
In this section, we provide a probabilistic interpretation of the optimization problem, which also serves as the motivation behind the definition of the dimensionality-normalized norm. Suppose every row $X_i$ in the matrix $\mathbf{X}$ is generated independently through the following process:
\begin{enumerate}
    \item Select a nonempty subset of all the columns $I_j \subset I$, and the entries in those columns follow a multivariate normal distribution with mean vector $c_j$ and covariance matrix $\sigma^2 \mathbf{I}$ (here $\mathbf{I}$ denotes the identity matrix):
    \begin{equation*}
        X_i(I_j) \sim \mathcal{N}(c_j, \sigma^2 \mathbf{I}).
    \end{equation*}
    \item The entries in the other columns are considered as noise and do not affect the likelihood of $X_i$.
\end{enumerate}
The log-likelihood of $X_i$ has the following property:
\begin{equation*}
    \log \mathcal{L}(I_j, c_j|X_i) \propto - \frac{1}{2\sigma^2}||X_i(I_j)-c_j||_2^2 - \frac{l_j}{2} \log(2\pi \sigma^2),
\end{equation*}
where $l_j = |I_j|$ is the cardinality of the index set $I_j$ and also the dimensionality of the vector $X_i(I_j)$.

Naturally, the next step is to maximize the log-likelihood of $X_i$ over $I_j$ and $c_j$. However, there is one issue: different $I_j$ might have different cardinality $l_j$. This means that $X_i(I_j)$ might have different dimensionality, and directly comparing the log-likelihood of vectors of different dimensionality is problematic: when $\log(2\pi \sigma^2) > 0$, increasing the dimensionality $l_j$ would monotonically decrease the log-likelihood.

One simple solution is to maximize the dimensionality-normalized log-likelihood of $X_i$:
\begin{equation*}
    \frac{\log \mathcal{L}(I_j, c_j|X_i)}{l_j} \propto - \frac{||X_i(I_j)-c_j||_2^2}{l_j},
\end{equation*}
which is equivalent to minimizing 
\begin{equation*}
    \frac{||X_i(I_j)-c_j||_2^2}{l_j} = ||X_i(I_j)-c_j||_{dn}^2.
\end{equation*}
Therefore, maximizing the joint dimensionality-normalized log-likelihood of all the rows in the matrix $\mathbf{X}$ is equivalent to minimizing the empirical clustering risk
\begin{equation*}
    \frac{1}{n} \sum_{i=1}^{n} \min_{1 \le j \le k} ||X_i(I_j) - c_j||_{dn}^2.
\end{equation*}
This equivalence establishes a connection between the optimization perspective and the probabilistic perspective of the biclustering problem.

\section{Algorithm}\label{sec:algorithm}
In this section, we present a simple and novel algorithm that finds a local minimum of the empirical clustering risk. An implementation of our algorithm is provided in the R package \texttt{akmbiclust}, available on CRAN.

The idea is to alternate the use of an adapted version of the $k$-means clustering algorithm between columns and rows. Similar to the widely used Lloyd's algorithm in $k$-means clustering, our algorithm is also based on heuristics and does not guarantee to achieve global optimum. In each individual run, our alternating $k$-means biclustering algorithm works as described in Algorithm \ref{algorithm} below.  

\begin{algorithm}
\begin{enumerate}
    \item Start by performing $k$-means clustering separately on the rows and columns of the input matrix $\mathbf{X}$ to obtain the initial partitions of the $n$ rows $J_1, \dots, J_k$ and $m$ columns $I_1, \dots, I_k$. Calculate and record the loss.
    \item With a fixed $I_1, \dots, I_k$, the optimal cluster centers $c_1, \dots, c_k$ could be found in the following way: 
    \begin{enumerate}
        \item (Update step) Given row partitions $J_1, \dots, J_k$, update the cluster centers $c_1, \dots, c_k$ by the following equation:
        \begin{equation*}
            c_j = \frac{1}{|J_j|} \sum_{i \in J_j} X_i(I_j), 1 \le j \le k.
        \end{equation*}
        \item (Assignment step) Given cluster centers $c_1, \dots, c_k$, update the row partitions $J_1, \dots, J_k$ by assigning every row to the cluster center with the smallest distance (induced by the dimensionality-normalized norm), and all the rows that are closest to $c_j$ form $J_j, 1 \le j \le k$.
    \end{enumerate}
    Alternate between (a) and (b) until convergence, and obtain a partition of the $n$ rows $J_1, \dots, J_k$.
    \item Transpose the matrix $\mathbf{X}$, and $J_1, \dots, J_k$ becomes a partition of the $n$ columns. Again, alternate between (a) and (b) until convergence, and obtain a partition the $m$ rows $I_1, \dots, I_k$. Transpose the matrix $\mathbf{X}$ back, and $I_1, \dots, I_k$ becomes a partition of the $m$ columns.
    \item Alternate between step 2 and step 3 until convergence. Calculate and record the loss.
    \item Compare the losses at the end of step 1 and step 4. Output the $k$ groups of row indices $J_1, \dots, J_k$ and column indices $I_1, \dots, I_k$ associated with the minimum loss.
\end{enumerate}
\caption{Alternating $k$-means biclustering}
\label{algorithm}
\end{algorithm}

Noticeably, the subroutine that alternates between (a) and (b) is quite similar to the widely used Lloyd's algorithm in $k$-means clustering, which is the reason why our algorithm is called ``alternating $k$-means biclustering''. However, we note that there are two important differences:
\begin{enumerate}
    \item The cluster centers $c_1, \dots, c_k$ are not vectors in $\mathbb{R}^m$. Instead, $c_j \in \mathbb{R}^{I_j}$ for $1 \le j \le k$.
    \item When calculating the distance between a row $X_i$ and a cluster center $c_j$, the distance function is not induced by the Euclidean norm. Instead, it is induced by the dimensionality-normalized norm.
\end{enumerate}

It is recommended to run our algorithm multiple times and choose the result with the minimum loss, each time starting with a different initialization by randomly permuting the rows and columns of the input matrix $\mathbf{X}$. The reason is twofold:
\begin{enumerate}
    \item First, our algorithm does not guarantee global optimum and depends on the partitions obtained from the initial separate $k$-means clustering, so running our algorithm multiple times increases the probability of finding a smaller local minimum.
    \item Second, just like $k$-means clustering algorithm, our algorithm also might encounter empty cluster problem. More specifically, if in any assignment step any group of row indices $J_j, 1 \le j \le k$ becomes empty, then the algorithm cannot proceed and need to restart. The probability of having empty clusters is small when $k$ is much smaller than $\min(n, m)$, but might become larger when $k$ approaches $\min(n, m)$.
\end{enumerate}

\section{Penalization}\label{sec:penalization}
In this section, we consider extending our method by adding penalization terms to the loss function. So far, the loss function that we minimize is the empirical clustering risk: 
\begin{equation*}
    \frac{1}{n} \sum_{i=1}^{n} \min_{1 \le j \le k} ||X_i(I_j) - c_j||_{dn}^2.
\end{equation*}
Intuitively, minimizing this loss encourages all the rows in the bicluster $j$ to have similar entries in the columns $I_j$ across different rows. However, it does not differentiate between how large or small those entries are. Therefore, it might be beneficial to place some form of penalization on those entries to encourage ``good'' biclustering results. To this end, we consider the following penalization method: 

Let $||\mathbf{X}||_{F} = \sqrt{\sum_{i=1}^{n} \sum_{j=1}^{m} X_{ij}^2}$ denote the Frobenius norm of a matrix $\mathbf{X} \in \mathbb{R}^{n \times m}$. Let the index of the $k$ biclusters be from 0 to $k-1$. For every bicluster $j$ where $0 \le j \le k-1$, let $\mathbf{X}^{(j)}$ denote the submatrix of $\mathbf{X}$ consisting only of rows and columns that belong to bicluster $j$. The following penalization term is added to every bicluster $j$ where $1 \le j \le k-1$: 
\begin{equation*}
    \lambda \cdot \frac{||\mathbf{X}||_{F}^2}{||\mathbf{X}^{(j)}||_{F}^2 + 1}.
\end{equation*}
Note that bicluster $0$ is the special bicluster that does not have the above penalization term. The parameter $\lambda$ is a tuning parameter.

The motivation of the penalization method comes from the following two observations:
\begin{enumerate}
    \item In general, the entries in the biclusters should represent signals, which means that they should not be close to zero. Therefore, the Frobenius norm of the submatrix $||\mathbf{X}^{(j)}||_{F}$ induced by the bicluster $j$ should be large.
    \item However, not all rows and columns should be classified into one of $k$ biclusters representing signals. Some rows or columns might just consist of random noise, and it is reasonable to include a special bicluster that represents random noise. Adding the penalization term $\lambda \cdot \frac{||\mathbf{X}||_{F}^2}{||\mathbf{X}^{(j)}||_{F}^2 + 1}$ to every bicluster creates a potential problem: when there is supposed to be a bicluster with entries close to zero, the penalization term might become excessively large. Thus, we choose to not apply the penalization term to bicluster $0$, which is the special bicluster representing noise.
\end{enumerate}
In the end, we decide to use the above penalization function, because it works reasonably well on both simulated and real-world data.

The penalized loss function can be written as:
\begin{equation*}
    \frac{1}{n} \sum_{i=1}^{n} \min_{0 \le j \le k-1} ||X_i(I_j) - c_j||_{dn}^2 + \lambda \sum_{j=1}^{k-1} \frac{||\mathbf{X}||_{F}^2}{||\mathbf{X}^{(j)}||_{F}^2 + 1}.
\end{equation*}
When $\lambda = 0$, the penalized loss function reduces to the empirical clustering risk. 

It is important to point out that the loss function does not affect the alternating process between step 2 and step 3 in Algorithm \ref{algorithm}. However, the loss function does affect which specific biclustering result is chosen among different biclustering results produced by the algorithm: typically the algorithm is run with many random initializations, and even in each individual run, the algorithm needs to compare the two losses of two biclustering results: one at the end of step 1 corresponding to the partitions obtained by separate $k$-means, one at the end of step 4 when the algorithm finishes performing alternating $k$-means. Among all the different biclustering results, the biclustering result with the minimum loss is chosen as the final output.

\section{Simulation Studies}\label{sec:simulation}
In this section, we evaluate and compare the performance of the following five biclustering methods on simulated data with different settings:
\begin{enumerate}
    \item Alternating $k$-means biclustering (AKM): This is the method presented in this paper. We use the penalized loss function in Section \ref{sec:penalization}, with three different $\lambda$ values: 0, 0.1, and 1.
    \item Separate $k$-means clustering (KM): This method simply performs $k$-means clustering separately on the rows and columns.
    \item Profile likelihood biclustering (PL) \citep{flynn2020profile}: This method is based on profile likelihood and has associated consistency guarantees. We implement the method using the R package \texttt{biclustpl}. The distribution family is selected as Gaussian, which is the true distribution of the simulated data.
    \item Sparse biclustering (SBC) \citep{tan2014sparse}: This method assumes the entries are normally distributed with a bicluster-specific mean and a common variance, and maximizes the $L_1$-penalized log-likelihood to obtain sparse biclusters. We implement the method using the R package \texttt{sparseBC}. The input matrix is always mean-centered before applying the method, and the tuning parameter $\lambda$ is selected by choosing the $\lambda$ with the smallest BIC over a grid of $\lambda$ values, both of which are suggested in their paper.
    \item High-order spectral clustering (HSC) \citep{han2020exact}: This method assumes a tensor block model and provides statistical optimality guarantees under a mild sub-Gaussian noise assumption. We implement the method using the R package \texttt{HLloyd}.
\end{enumerate}
All methods except HSC are run with 100 random initializations. 

Among numerous existing biclustering methods, the above five methods are selected to evaluate and compare their performance in the simulation studies because they satisfy the following two requirements:
\begin{enumerate}
    \item Every row should be classified into one and only one row cluster. In addition, every column should also be classified into one and only one column cluster. This means that the biclustering methods should produce non-overlapping biclusters with checkerboard structure (Figure \ref{checkerboard}) or block-diagonal biclusters (Figure \ref{exclusive}).
    \item In addition, the biclustering methods should also allow explicitly specifying the number of clusters that the rows and columns are classified into. A few biclustering methods such as spectral biclustering \citep{kluger2003spectral}, SSVD \citep{lee2010biclustering}, and convex biclustering \citep{chi2017convex} satisfy the first requirement but do not satisfy this requirement.  
\end{enumerate}

In our simulations, the evaluation metric is the misclassification rate, which is defined as:
\begin{equation*}
    \text{misclassification rate} = \frac{\text{number of entries classified into the wrong row or column cluster}}{\text{total number of entries in the input matrix }\mathbf{X}}.
\end{equation*}
Smaller misclassification rate indicates better performance, and a perfect biclustering result would have a misclassification rate of 0.

We generate simulated data in three different settings. In all three settings, the input matrix $\mathbf{X}$ is generated using a $2 \times 2$ block model, and for all the methods we set the number of clusters that the rows and columns are classified into to be 2. The number of rows is set to be $n = 400$, and the number of columns is set to be $m = a \cdot n$ where $a \in \{0.5, 1.0, 2.0\}$. The entries $X_{ij}$ in the input matrix $\mathbf{X}$ are generated independently through the following process:
\begin{enumerate}
    \item Sample the true row class $u_i \in \{1, 2\}$ from the multinomial distribution with probability $p = (0.3, 0.7)$.
    \item Sample the true column class $v_j \in \{1, 2\}$ from the multinomial distribution with probability $q = (0.2, 0.8)$.
    \item Conditioning on $u_i$ and $v_j$, $X_{ij}$ follows a Gaussian distribution with mean $\mathbf{M}_{u_iv_j}$ and standard deviation $\mathbf{S}_{u_iv_j}$:
    \begin{equation*}
        X_{ij}|u_i,v_j \sim \mathcal{N}(\mathbf{M}_{u_iv_j}, \mathbf{S}_{u_iv_j}^2).
    \end{equation*}
\end{enumerate}
Note that $\mathbf{M}$ and $\mathbf{S}$ are $2 \times 2$ matrices representing the means and standard deviations of entries in different blocks. They are different for each simulation setting.

\subsection{Simulation 1: Blocks with Different Means and the Same Variance}
In the first simulation, we consider the case where the $2 \times 2$ blocks have different means and the same variance. More specifically, we set 
\begin{equation*}
    \mathbf{M} = b \cdot
    \begin{bmatrix}
    0.36 & 0.90 \\
    -0.58 & -0.06
    \end{bmatrix},
\end{equation*}
where $b \in \{0.20, 0.25, 0.30\}$. The entries of the matrix are simulated from a uniform distribution on $[-1, 1]$. As $b$ increases, the difference between the means in different blocks also increases. In addition, we set
\begin{equation*}
    \mathbf{S} = 
    \begin{bmatrix}
    1 & 1 \\
    1 & 1
    \end{bmatrix},
\end{equation*}
which means that all entries have the same standard deviation of 1. This type of structure is exactly what many biclustering methods including PL, SBC, and HSC assume the input matrix $\mathbf{X}$ has, therefore we would expect their performance to be good.

\begin{table}[ht]
\centering
\resizebox{\columnwidth}{!}{
\begin{tabular}{ c c c c c c c c }
    \hline
    & AKM ($\lambda = 0$) & AKM ($\lambda = 0.1$) & AKM ($\lambda = 1$) & KM & PL & SBC & HSC \\
    \hline
    & & & & $b = 0.20$ & & & \\
    \hline
    $a = 0.5$ & 0.722(0.002) & 0.591(0.015) & 0.527(0.007) & 0.542(0.006) & 0.328(0.010) & 0.409(0.014) & 0.387(0.010) \\
    $a = 1.0$ & 0.723(0.002) & 0.470(0.006) & 0.474(0.006) & 0.489(0.006) & 0.244(0.005) & 0.258(0.008) & 0.313(0.006) \\
    $a = 2.0$ & 0.718(0.002) & 0.445(0.006) & 0.447(0.006) & 0.444(0.006) & 0.221(0.004) & 0.222(0.004) & 0.299(0.004) \\
    \hline
    & & & & $b = 0.25$ & & & \\
    \hline
    $a = 0.5$ & 0.715(0.003) & 0.457(0.010) & 0.450(0.007) & 0.467(0.007) & 0.201(0.007) & 0.223(0.012) & 0.251(0.007) \\
    $a = 1.0$ & 0.716(0.003) & 0.415(0.010) & 0.415(0.010) & 0.404(0.007) & 0.150(0.006) & 0.151(0.006) & 0.222(0.006) \\
    $a = 2.0$ & 0.706(0.002) & 0.347(0.013) & 0.338(0.013) & 0.321(0.006) & 0.147(0.004) & 0.148(0.004) & 0.223(0.005) \\
    \hline
    & & & & $b = 0.30$ & & & \\
    \hline
    $a = 0.5$ & 0.711(0.003) & 0.392(0.012) & 0.387(0.011) & 0.386(0.010) & 0.110(0.004) & 0.110(0.004) & 0.166(0.006) \\
    $a = 1.0$ & 0.694(0.003) & 0.291(0.015) & 0.270(0.013) & 0.287(0.010) & 0.083(0.004) & 0.079(0.004) & 0.140(0.006) \\
    $a = 2.0$ & 0.670(0.004) & 0.173(0.007) & 0.169(0.005) & 0.197(0.007) & 0.080(0.003) & 0.081(0.003) & 0.143(0.006) \\
    \hline
\end{tabular}
}
\caption{The means (and standard errors) of the misclassification rate for Simulation 1 over 50 simulations.}
\label{tab:sim1}
\end{table}

Results are reported in Table \ref{tab:sim1}. Under this setting, we see that PL and SBC have similar misclassification rates, followed by HSC, and all three are much smaller than the other four methods. KM, AKM with $\lambda = 1$ and $\lambda = 0.1$ also have similar misclassification rates, though they are significantly larger than PL, SBC, and HSC. AKM with $\lambda = 0$ has the worst performance, with misclassification rates around 0.7 in all cases. In addition, we observe a general trend that as $a$ and $b$ increase, the misclassification rates decrease. This trend agrees with our expectation, because larger $a$ means larger input matrix, and larger $b$ means larger difference between the means in different blocks, both of which should improve the performance of biclustering methods.

\subsection{Simulation 2: Blocks with Different Variances and the Same Mean}
In the second simulation, we consider the case where the $2 \times 2$ blocks have different variances and the same mean. More specifically, we set 
\begin{equation*}
    \mathbf{S} = 
    \begin{bmatrix}
    1+b & 1 \\
    1 & 1+b
    \end{bmatrix},
\end{equation*}
where $b \in \{0.20, 0.25, 0.30\}$. As $b$ increases, the difference between the standard deviations in different blocks also increases. In addition, we set
\begin{equation*}
    \mathbf{M} = 
    \begin{bmatrix}
    0 & 0 \\
    0 & 0
    \end{bmatrix},
\end{equation*}
which means that all entries have the same mean of 0. This is the case where the blocks are defined not by different means but by different variances, and many biclustering methods including PL, SBC, and HSC are unable to detect this type of structure.

\begin{table}[ht]
\centering
\resizebox{\columnwidth}{!}{
\begin{tabular}{ c c c c c c c c }
    \hline
    & AKM ($\lambda = 0$) & AKM ($\lambda = 0.1$) & AKM ($\lambda = 1$) & KM & PL & SBC & HSC \\
    \hline
    & & & & $b = 0.20$ & & & \\
    \hline
    $a = 0.5$ & 0.475(0.015) & 0.650(0.013) & 0.700(0.003) & 0.722(0.002) & 0.726(0.002) & 0.717(0.006) & 0.727(0.002) \\
    $a = 1.0$ & 0.082(0.013) & 0.357(0.029) & 0.706(0.002) & 0.730(0.002) & 0.727(0.001) & 0.723(0.006) & 0.731(0.001) \\
    $a = 2.0$ & 0.012(0.001) & 0.150(0.014) & 0.719(0.002) & 0.731(0.002) & 0.730(0.001) & 0.721(0.008) & 0.733(0.001) \\
    \hline
    & & & & $b = 0.25$ & & & \\
    \hline
    $a = 0.5$ & 0.139(0.019) & 0.372(0.031) & 0.698(0.003) & 0.724(0.002) & 0.726(0.002) & 0.724(0.002) & 0.724(0.002) \\
    $a = 1.0$ & 0.008(0.001) & 0.088(0.008) & 0.715(0.002) & 0.730(0.002) & 0.732(0.001) & 0.724(0.006) & 0.731(0.002) \\
    $a = 2.0$ & 0.002(0.000) & 0.050(0.003) & 0.720(0.002) & 0.733(0.001) & 0.729(0.002) & 0.728(0.006) & 0.732(0.001) \\
    \hline
    & & & & $b = 0.30$ & & & \\
    \hline
    $a = 0.5$ & 0.019(0.002) & 0.099(0.015) & 0.701(0.003) & 0.724(0.002) & 0.727(0.002) & 0.714(0.008) & 0.725(0.002) \\
    $a = 1.0$ & 0.002(0.000) & 0.036(0.002) & 0.717(0.002) & 0.731(0.002) & 0.728(0.001) & 0.726(0.005) & 0.729(0.002) \\
    $a = 2.0$ & 0.000(0.000) & 0.022(0.002) & 0.722(0.002) & 0.735(0.001) & 0.729(0.002) & 0.723(0.008) & 0.731(0.002) \\
    \hline
\end{tabular}
}
\caption{The means (and standard errors) of the misclassification rate for Simulation 2 over 50 simulations.}
\label{tab:sim2}
\end{table}

Results are reported in Table \ref{tab:sim2}. Under this setting, we see that KM, PL, SBC, HSC, and AKM with $\lambda = 1$ all have similarly bad performance, with misclassification rates around 0.72 in all cases. AKM with $\lambda = 0$ achieves the smallest misclassification rates in all cases, and in some cases ($a \in \{1.0, 2.0\}, b \in \{0.25, 0.30\}$) even produces near perfect biclustering results. AKM with $\lambda = 0.1$ also has good performance, with slightly larger misclassification rates compared to AKM with $\lambda = 0$. In addition, with regard to AKM with $\lambda = 0$ and $\lambda = 0.1$, we also see the general trend that as $a$ and $b$ increase, the misclassification rates decrease. Interestingly, in this setting larger $b$ means larger difference between the variances in different blocks, and only AKM with $\lambda = 0$ and $\lambda = 0.1$ are able to detect and respond to this type of structure.

\subsection{Simulation 3: Blocks with Different Means and Different Variances}
In the third simulation, we consider the case where the $2 \times 2$ blocks have different means and different variances, which is a combination of the first and second case. More specifically, we set 
\begin{equation*}
    \mathbf{M} = b \cdot
    \begin{bmatrix}
    0.36 & 0.90 \\
    -0.58 & -0.06
    \end{bmatrix},\
    \mathbf{S} = 
    \begin{bmatrix}
    1+b & 1 \\
    1 & 1+b
    \end{bmatrix},
\end{equation*}
where $b \in \{0.20, 0.25, 0.30\}$. As $b$ increases, the difference between the means and standard deviations in different blocks also increases. This type of structure is arguably the most common type in practice, where different biclusters not only have different means but also have different variances.

\begin{table}[ht]
\centering
\resizebox{\columnwidth}{!}{
\begin{tabular}{ c c c c c c c c }
    \hline
    & AKM ($\lambda = 0$) & AKM ($\lambda = 0.1$) & AKM ($\lambda = 1$) & KM & PL & SBC & HSC \\
    \hline
    & & & & $b = 0.20$ & & & \\
    \hline
    $a = 0.5$ & 0.319(0.018) & 0.494(0.024) & 0.632(0.008) & 0.612(0.005) & 0.422(0.010) & 0.516(0.011) & 0.466(0.011) \\
    $a = 1.0$ & 0.031(0.004) & 0.183(0.015) & 0.535(0.006) & 0.542(0.004) & 0.320(0.005) & 0.334(0.008) & 0.333(0.007) \\
    $a = 2.0$ & 0.009(0.001) & 0.143(0.019) & 0.483(0.003) & 0.479(0.004) & 0.279(0.004) & 0.278(0.003) & 0.300(0.004) \\
    \hline
    & & & & $b = 0.25$ & & & \\
    \hline
    $a = 0.5$ & 0.042(0.006) & 0.132(0.012) & 0.540(0.005) & 0.543(0.006) & 0.289(0.009) & 0.345(0.014) & 0.311(0.009) \\
    $a = 1.0$ & 0.004(0.000) & 0.075(0.016) & 0.496(0.004) & 0.477(0.005) & 0.235(0.006) & 0.243(0.007) & 0.229(0.005) \\
    $a = 2.0$ & 0.001(0.000) & 0.092(0.021) & 0.474(0.003) & 0.416(0.006) & 0.215(0.003) & 0.217(0.003) & 0.221(0.005) \\
    \hline
    & & & & $b = 0.30$ & & & \\
    \hline
    $a = 0.5$ & 0.011(0.001) & 0.051(0.010) & 0.504(0.005) & 0.493(0.006) & 0.217(0.008) & 0.237(0.011) & 0.212(0.006) \\
    $a = 1.0$ & 0.001(0.000) & 0.070(0.020) & 0.481(0.003) & 0.425(0.006) & 0.167(0.006) & 0.175(0.007) & 0.153(0.005) \\
    $a = 2.0$ & 0.000(0.000) & 0.062(0.019) & 0.439(0.008) & 0.352(0.005) & 0.161(0.003) & 0.162(0.004) & 0.154(0.005) \\
    \hline
\end{tabular}
}
\caption{The means (and standard errors) of the misclassification rate for Simulation 3 over 50 simulations.}
\label{tab:sim3}
\end{table}

Results are reported in Table \ref{tab:sim3}. Under this setting, we see that KM and AKM with $\lambda = 1$ have similar and the worst performance. PL, SBC, and HSC also have similar but slightly better performance compared to KM and AKM with $\lambda = 1$. Most importantly, we see again that AKM with $\lambda = 0$ significantly outperforms all other methods in all cases, and has misclassification rates less than $0.05$ in all cases except when $a = 0.50$ and $b = 0.20$. In addition, AKM with $\lambda = 0.1$ also performs much better than KM, PL, SBC, and HSC in all cases except when $a = 0.50$ and $b = 0.20$, with slightly larger misclassification rates compared to AKM with $\lambda = 0$.

Comparing the results in Table \ref{tab:sim3} to those in Table \ref{tab:sim1}, we see that the additional difference between the variances in different blocks significantly benefits AKM with $\lambda = 0$ and $\lambda = 0.1$, resulting in a drastic decrease in misclassification rates. In contrast, this additional difference between block variances harms the performance of AKM with $\lambda = 1$, KM, PL, SBC, and HSC, causing varying degrees of increase in misclassification rates. This indicates that AKM with $\lambda = 0$ and $\lambda = 0.1$ are capable of leveraging the information about difference between block variances to detect meaningful structures, whereas AKM with $\lambda = 1$, KM, PL, SBC, and HSC are adversely affected by difference between block variances.

Comparing the results in Table \ref{tab:sim3} to those in Table \ref{tab:sim2}, we see that the additional difference between the means in different blocks significantly benefits AKM with $\lambda = 1$, KM, PL, SBC, and HSC, which is expected because many biclustering methods including PL, SBC, and HSC make the explicit assumption that different biclusters should have different means. However, this additional difference between block means also benefits AKM with $\lambda = 0$ and $\lambda = 0.1$, resulting in even better performance. 

Importantly, in Simulation 3, larger $b$ means larger difference between both the means and the variances in different blocks, so there are two kinds of signals present. In this situation, AKM with $\lambda = 0$ and $\lambda = 0.1$ perform much better than KM, PL, SBC, and HSC, indicating the possibility that AKM with appropriate $\lambda$ is suitable for dealing with complex data sets in the real world.

A visual illustration of different biclustering methods is provided in Figure \ref{fig:simulation_heatmaps}, which clearly demonstrates that AKM with $\lambda = 0$ and $\lambda = 0.1$ outperform other competing biclustering methods.

\begin{figure}[ht]
\centering

\begin{subfigure}{.24\textwidth}
\centering
\includegraphics[width=1\linewidth]{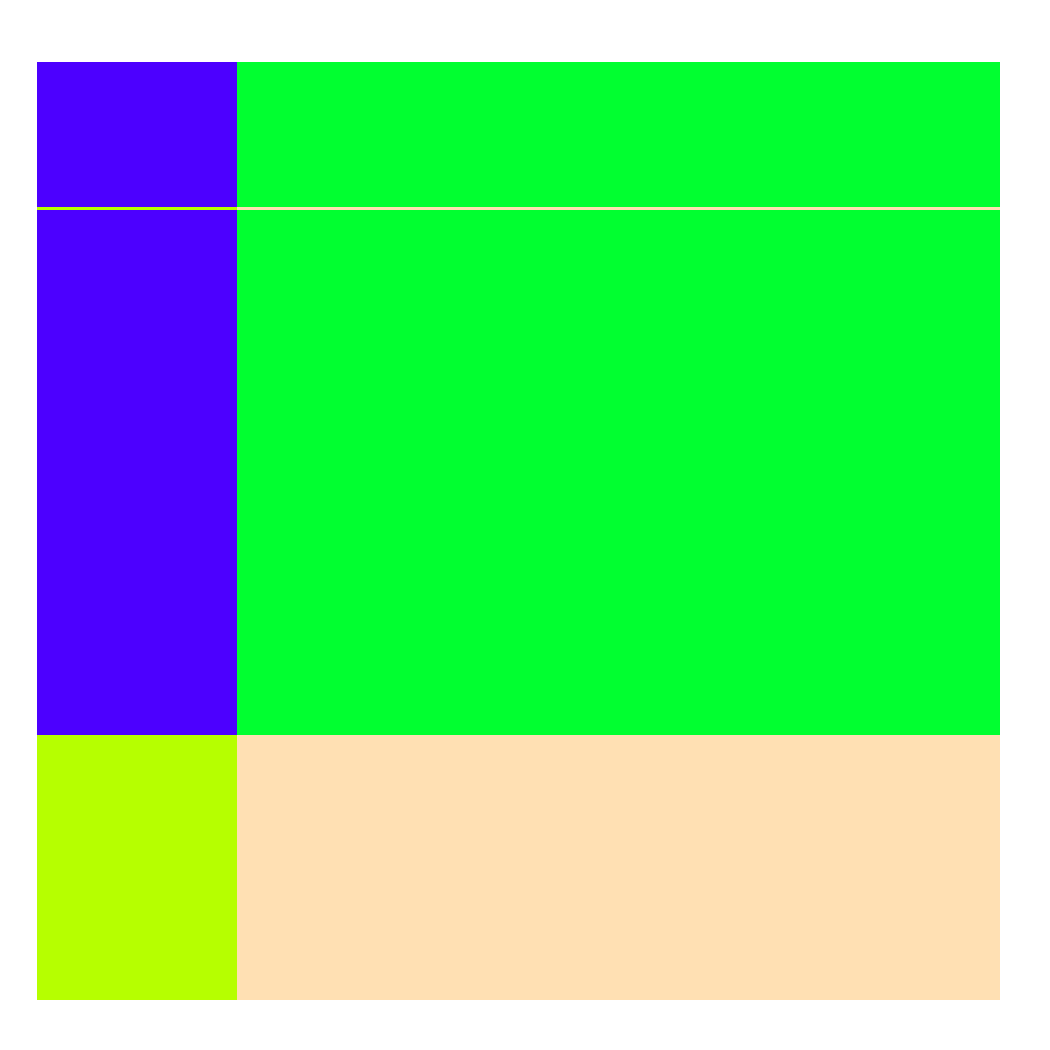}
\caption{}
\label{akm1_heatmap}
\end{subfigure}
\begin{subfigure}{.24\textwidth}
\centering
\includegraphics[width=1\linewidth]{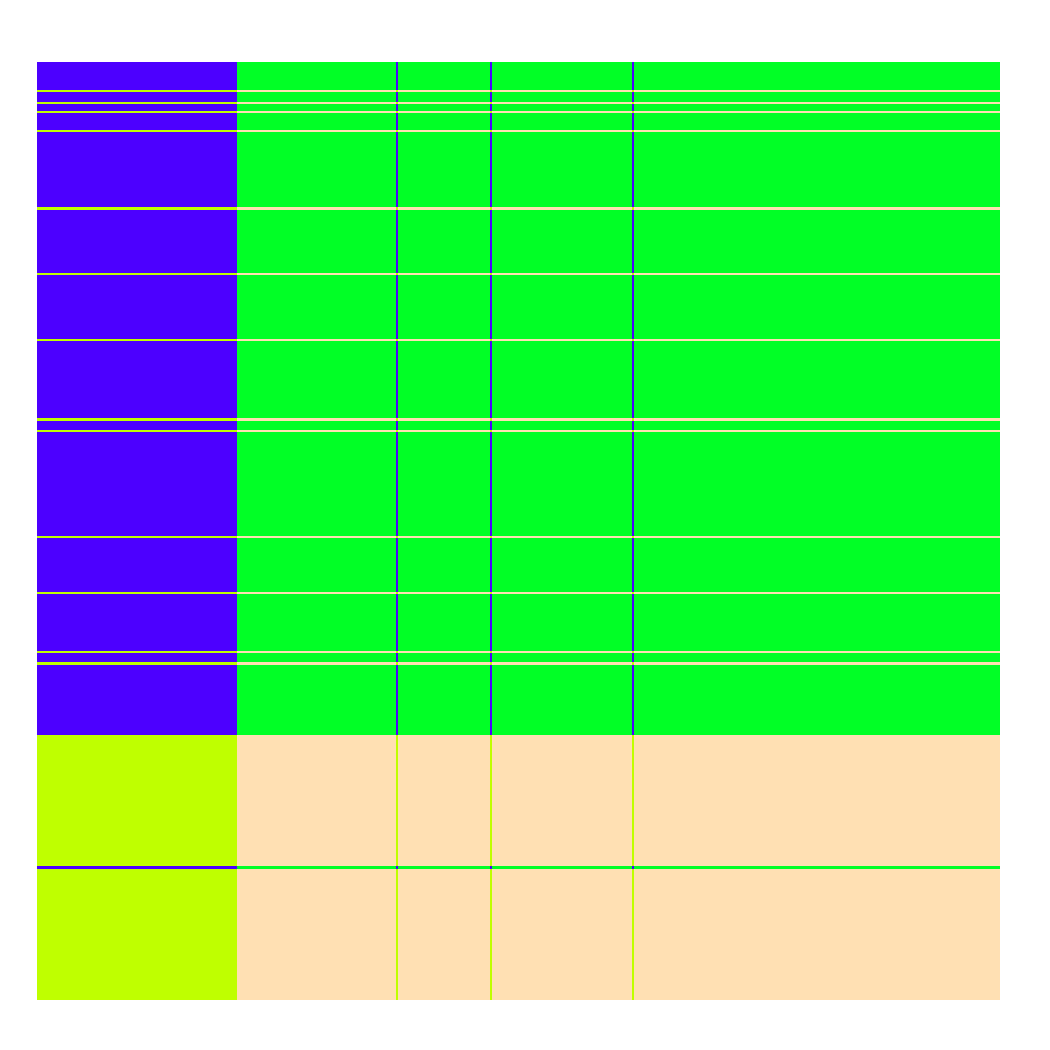}
\caption{}
\label{akm2_heatmap}
\end{subfigure}
\begin{subfigure}{.24\textwidth}
\centering
\includegraphics[width=1\linewidth]{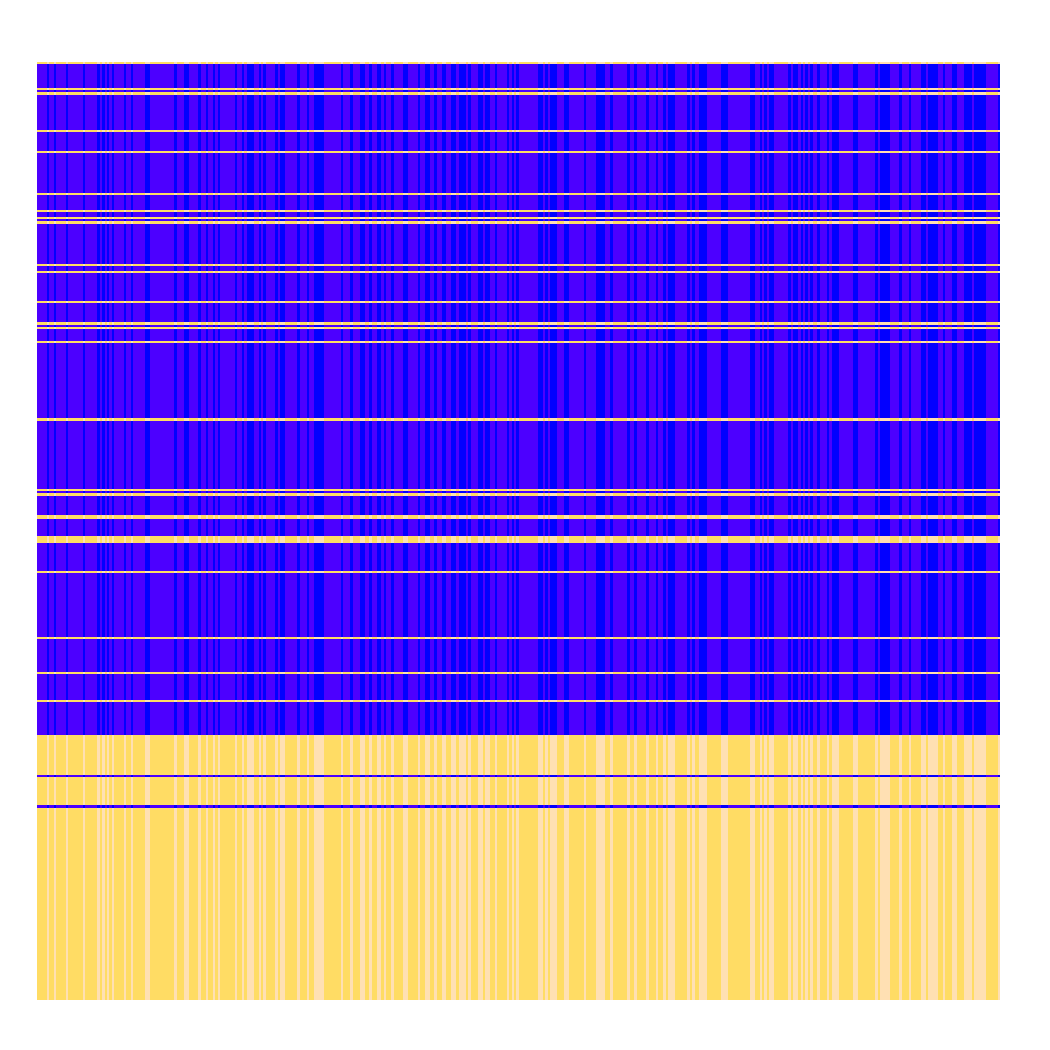}
\caption{}
\label{akm3_heatmap}
\end{subfigure}
\begin{subfigure}{.24\textwidth}
\centering
\includegraphics[width=1\linewidth]{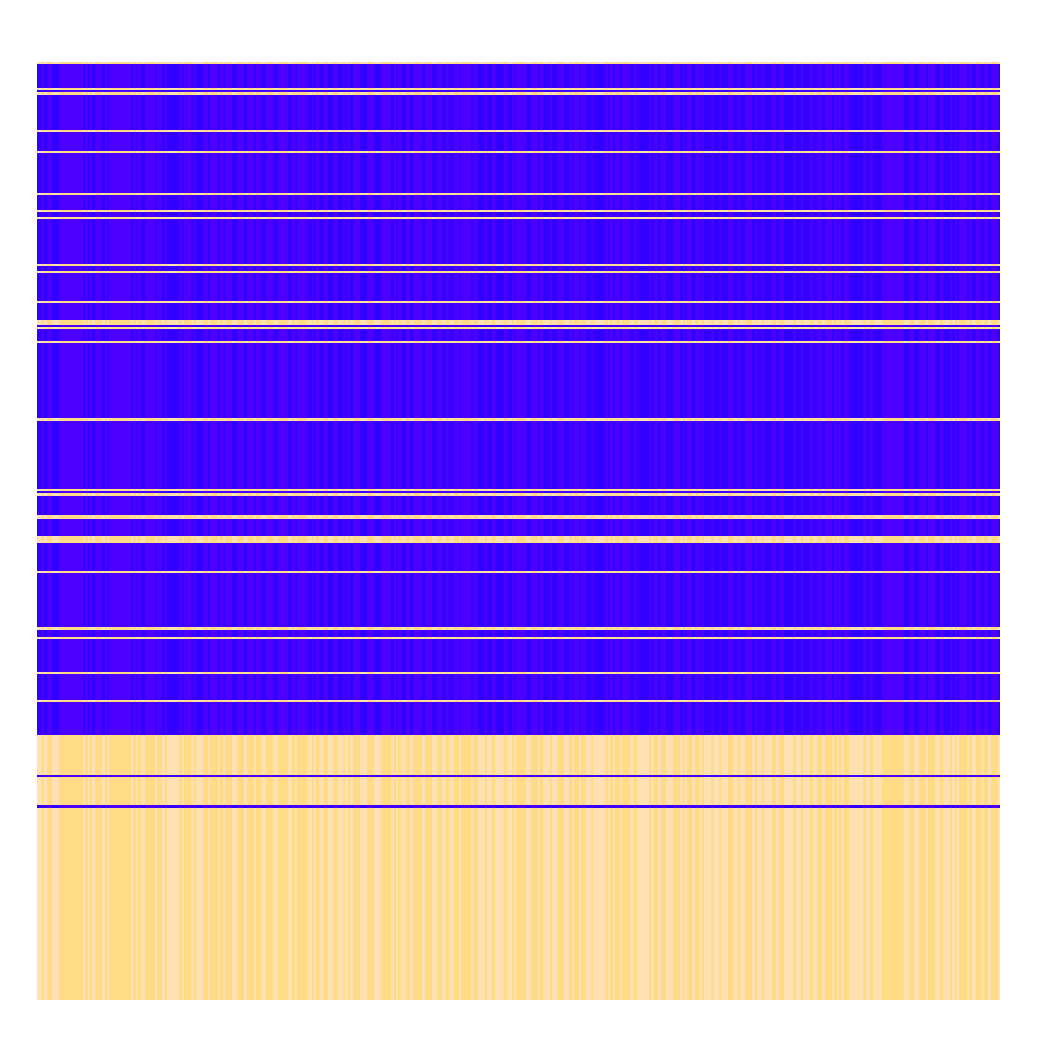}
\caption{}
\label{km_heatmap}
\end{subfigure}

\medskip

\begin{subfigure}{.24\textwidth}
\centering
\includegraphics[width=1\linewidth]{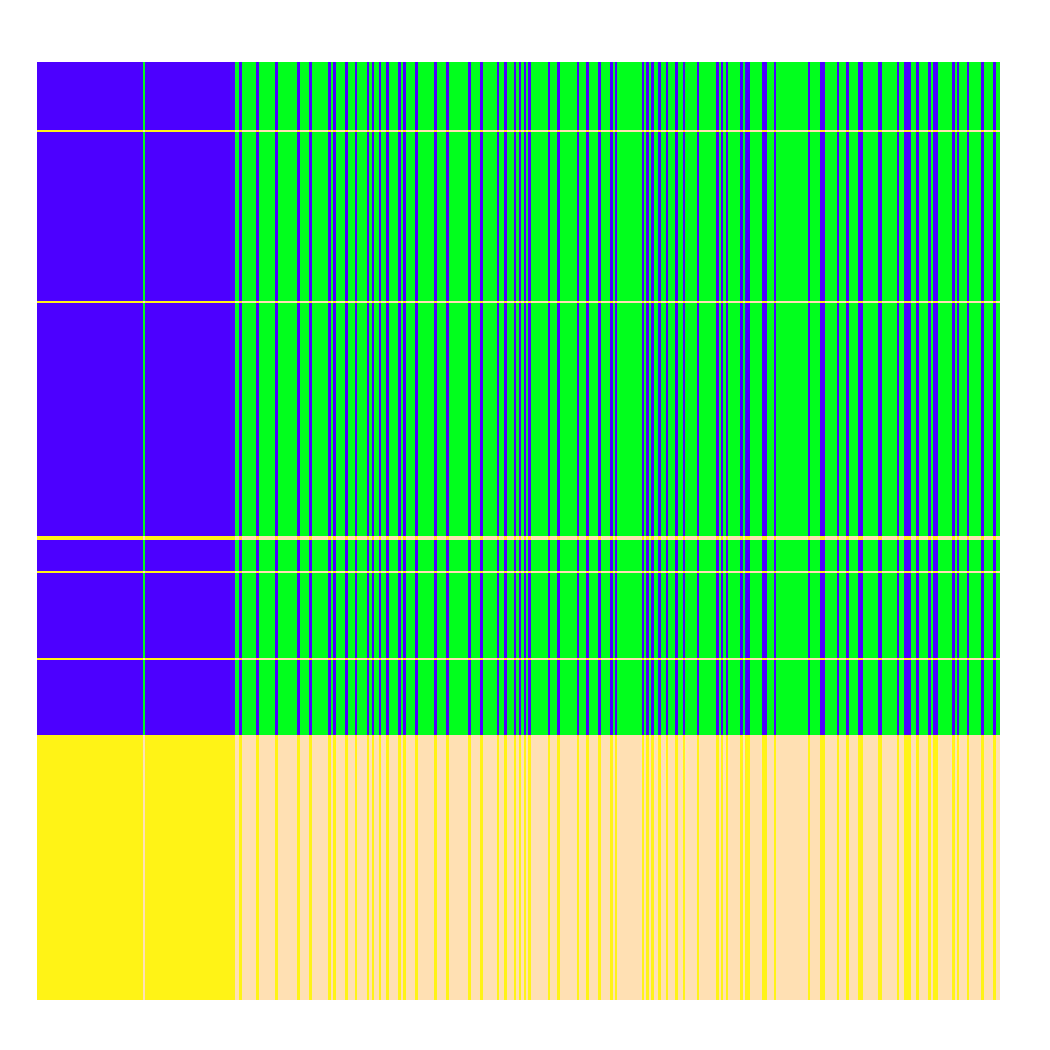}
\caption{}
\label{pl_heatmap}
\end{subfigure}
\begin{subfigure}{.24\textwidth}
\centering
\includegraphics[width=1\linewidth]{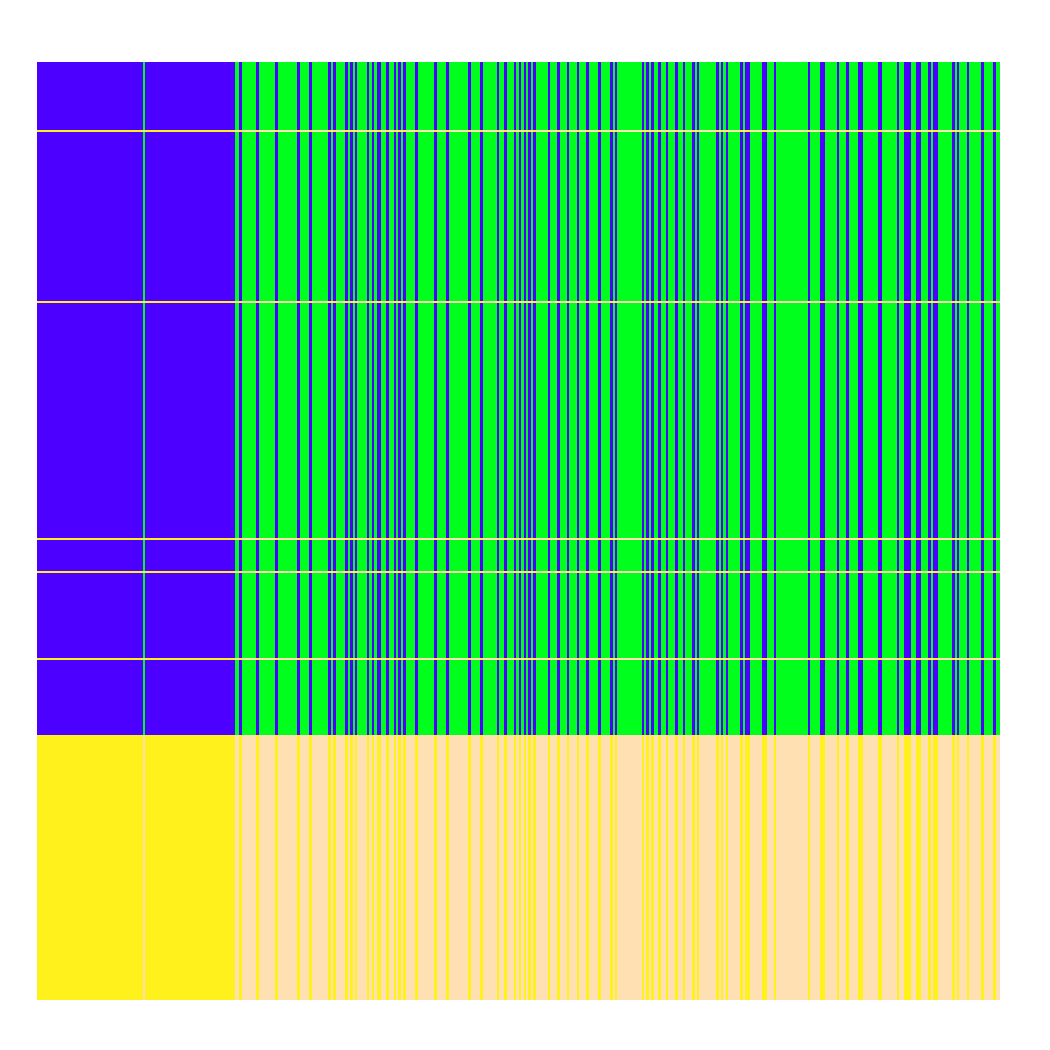}
\caption{}
\label{sbc_heatmap}
\end{subfigure}
\begin{subfigure}{.24\textwidth}
\centering
\includegraphics[width=1\linewidth]{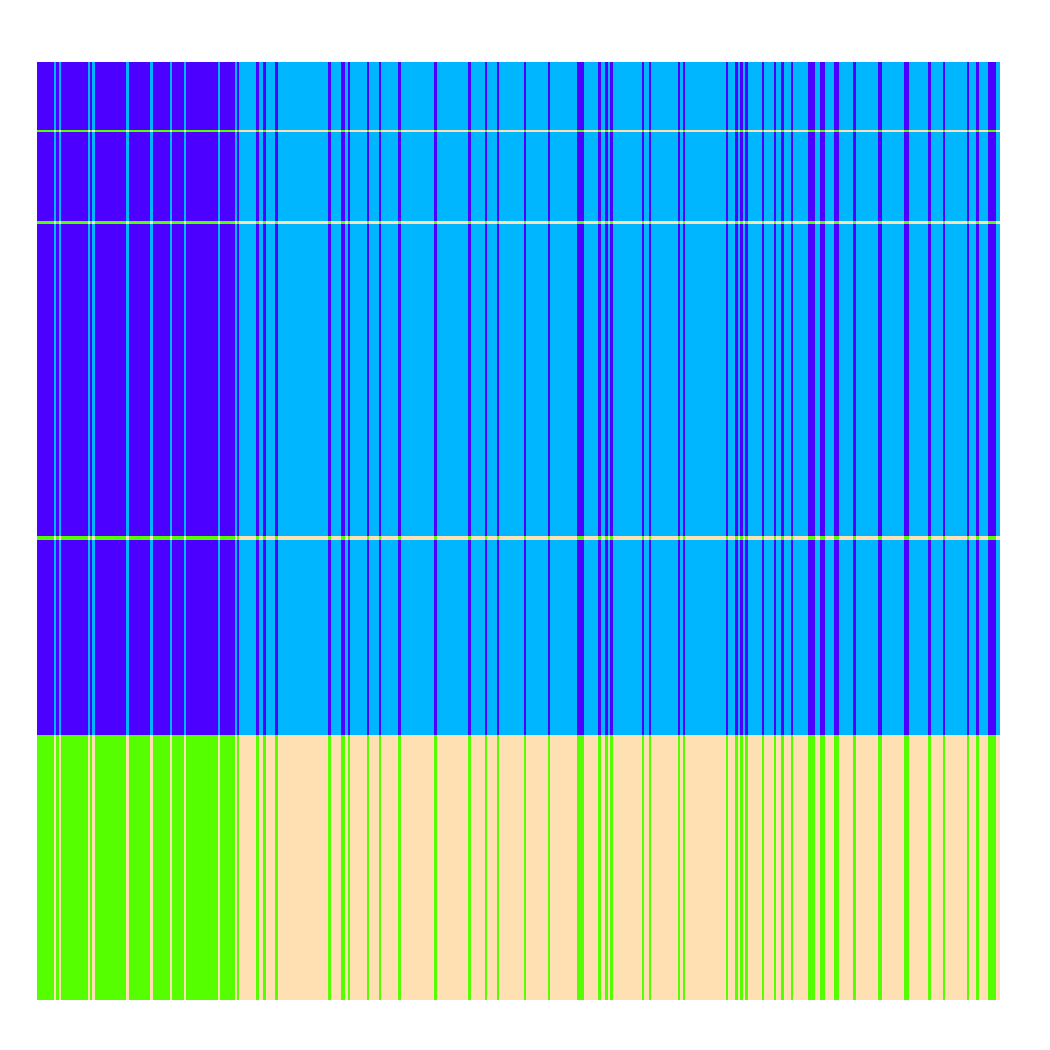}
\caption{}
\label{hsc_heatmap}
\end{subfigure}

\caption{Heatmaps of the estimated mean matrices from different biclustering methods for Simulation 3 with $a=1.0$ and $b=0.25$. The rows and columns are reordered so that the true row and column classes form $2 \times 2$ contiguous blocks. (a) AKM with $\lambda=0$, (b) AKM with $\lambda=0.1$, (c) AKM with $\lambda=1$, (d) KM, (e) PL, (f) SBC, and (g) HSC.}
\label{fig:simulation_heatmaps}
\end{figure}

\section{Applications}\label{sec:applications}
In this section, we apply our algorithm to three cancer gene expression data sets, all of which were proposed and preprocessed by \citet{de2008clustering}. In all three data sets, the rows represent different samples of tissues, and the columns represent different genes. The samples have already been classified into different groups based on their types of tissue, which means that the true sample cluster labels are available. This enables us to evaluate and compare the performance of our algorithm with four other biclustering algorithms (KM, PL, SBC, and HSC) in terms of sample misclassification rate, which is defined as 
\begin{equation*}
    \text{sample misclassification rate} = \frac{\text{number of samples classified into the wrong cluster}}{\text{total number of samples}}.
\end{equation*}
Smaller sample misclassification rate indicates better performance at clustering samples.

\subsection{Breast and Colon Cancer Gene Expression Data Set}
The first data set consists of 104 samples and 182 genes. There are only two types of samples: 62 samples correspond to breast cancer tissues, and 42 samples correspond to colon cancer tissues.

When applying our biclustering algorithm to real-world data, sometimes we do not have prior knowledge about the appropriate number of biclusters $k$. In that case, one good way to select $k$ is the ``elbow method'', which is also a widely used heuristic method to determine the number of clusters $k$ in traditional $k$ means clustering. The idea is to run the algorithm with $\lambda = 0$ and calculate the loss for different values of $k$, make a plot with loss on the $y$-axis and $k$ on the $x$-axis, and select the $k$ at the point of inflection (the ``elbow'' of the curve). In Figure \ref{fig:bc}, we plot the losses for $k$ from 1 to 10 when applying our algorithm with $\lambda = 0$ to the breast and colon cancer data set. By looking at Figure \ref{fig:bc}, it is quite clear that our algorithm should select $k=2$ as the number of biclusters, which is also the true number of row clusters.

\begin{figure}[ht]
\centering
\includegraphics[width=0.6\textwidth, height=0.4\textwidth]{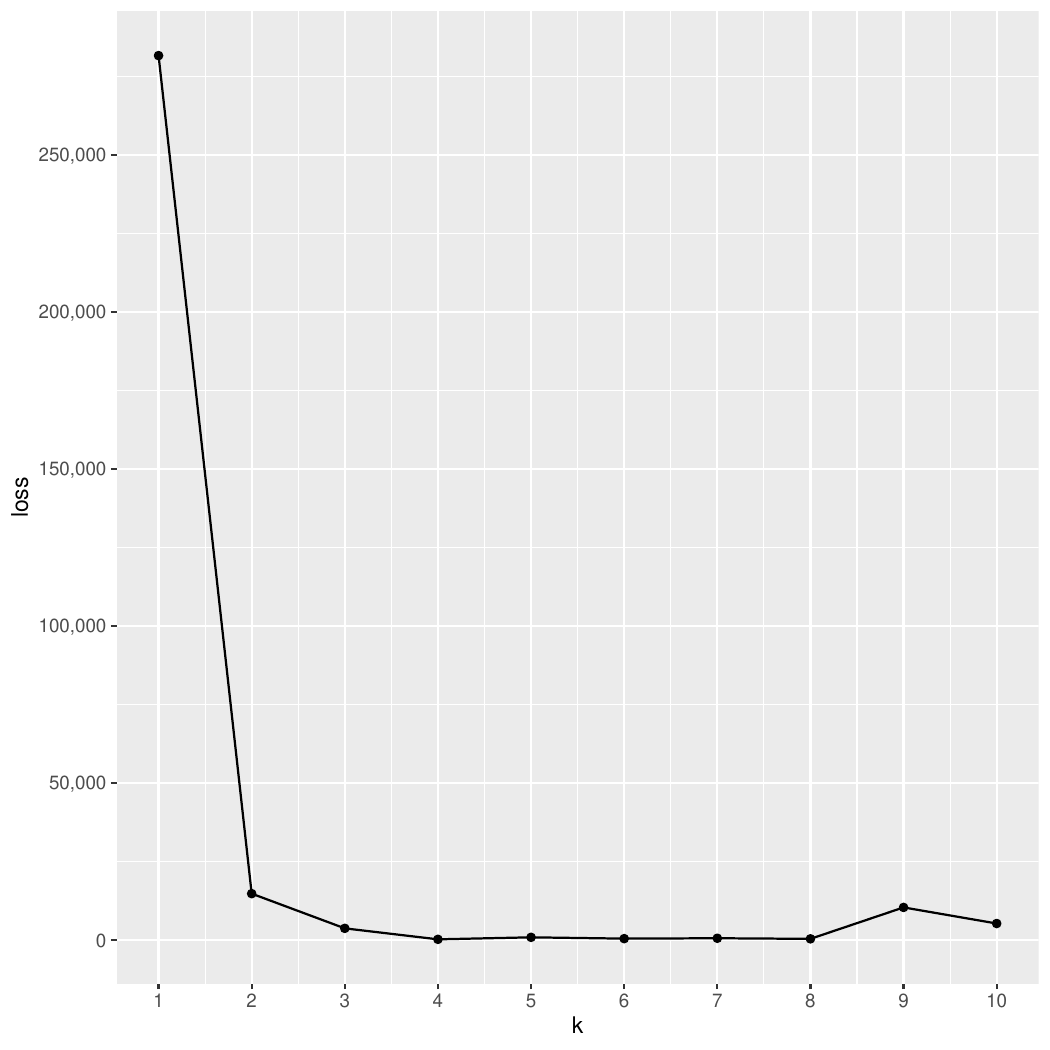}
\caption{Different losses for $k$ from 1 to 10 on the breast and colon cancer data set.}
\label{fig:bc}
\end{figure}

Having selected $k=2$ as the number of biclusters, we apply our algorithm with three different $\lambda$ values: 0, 0.1 and 1. For comparison, we also apply $k$-means clustering on the rows (KM), profile likelihood biclustering (PL), sparse biclustering (SBC), and high-order spectral clustering (HSC), with the number of row clusters set to 2. Since PL, SBC, and HSC all allow the number of column clusters to be different from the number of row clusters, we vary the number of column clusters from 1 to 20 and report the best result. For PL, the distribution family is selected as Gaussian. For SBC, the input matrix is always mean-centered before applying the method, and the tuning parameter $\lambda$ is selected by choosing the $\lambda$ with the smallest BIC over a grid of $\lambda$ values, both of which are suggested in their paper. All methods except HSC are run with 100 random initializations.

\begin{table}[ht]
\centering
\begin{tabular}{c c c c c c c}
    \hline
    AKM ($\lambda = 0$) & AKM ($\lambda = 0.1$) & AKM ($\lambda = 1$) & KM & PL & SBC & HSC \\
    \hline 
    0.0385 & 0.0385 & 0.0385 & 0.3462 & 0.3462 & 0.3462 & 0.3462 \\
    \hline
\end{tabular}
\caption{The sample misclassification rates on the breast and colon cancer data set.}
\label{tab:bc}
\end{table}

The sample misclassification rates are reported in Table \ref{tab:bc}. Noticeably, all four other biclustering methods have the same sample misclassification rate of 0.3462, which is around nine times larger than the sample misclassification rate of 0.0385, achieved by our algorithm with all three different values of $\lambda$. Compared to KM which ignores the interaction between samples and genes, our algorithm successfully leverages information about the interaction to significantly improve sample clustering performance. In contrast, PL, SBC, and HSC all fail to perform better than KM at clustering samples on this data set.

\subsection{Brain Cancer Gene Expression Data Set}
The second data set consists of 50 samples and 1739 genes. There are three types of samples: 31, 14, and 5 samples correspond to three different types of brain cancer tissues.

Again, we try to use the elbow method to select the appropriate number of biclusters $k$ for our algorithm. In Figure \ref{fig:brain}, we plot the losses for $k$ from 1 to 10 when applying our algorithm with $\lambda = 0$ to the brain cancer data set. In this case, it is not completely clear which $k$ we should select, although $k=2$ or $k=8$ might be the two most reasonable choices based on the plot alone. However, in some applications we do have prior knowledge about the appropriate number of biclusters $k$. For example, we know $k=3$ should be the default choice for this data set, because three row clusters could be naturally defined based on the three types of samples. Moreover, even if we select $k=2$ or $k=8$ based on the plot, the resulting biclusters could reveal interesting findings about the subgroups of samples or genes in this data set.

\begin{figure}[ht]
\centering
\includegraphics[width=0.6\textwidth, height=0.4\textwidth]{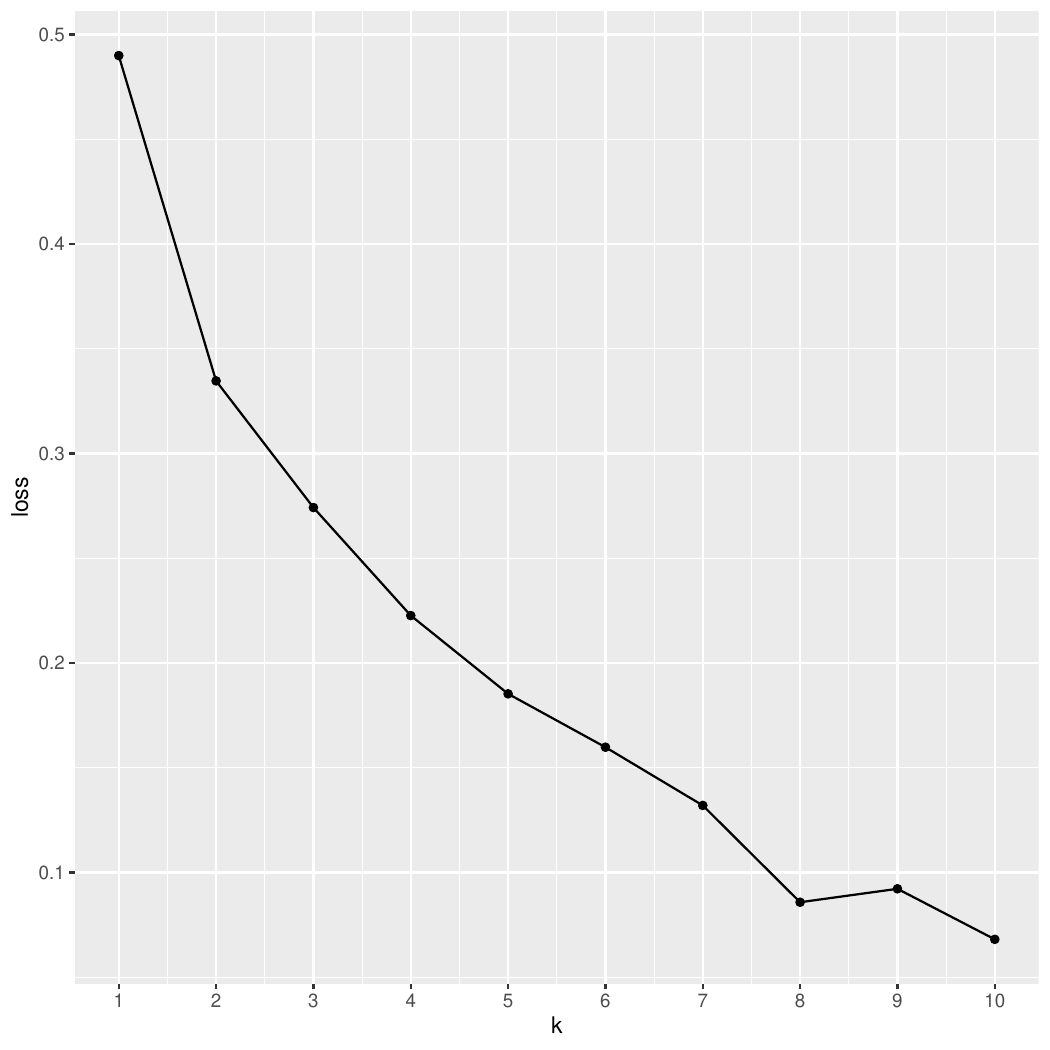}
\caption{Different losses for $k$ from 1 to 10 on the brain cancer data set.}
\label{fig:brain}
\end{figure}

Having selected $k=3$ as the number of biclusters, we again apply our algorithm with three different $\lambda$ values: 0, 0.1 and 1. We also apply KM, PL, SBC, and HSC, with the number of row clusters set to 3. All other settings of the biclustering algorithms are the same as in the first application.

\begin{table}[ht]
\centering
\begin{tabular}{c c c c c c c}
    \hline
    AKM ($\lambda = 0$) & AKM ($\lambda = 0.1$) & AKM ($\lambda = 1$) & KM & PL & SBC & HSC \\
    \hline 
    0.22 & 0.22 & 0.22 & 0.36 & 0.34 & 0.34 & 0.36 \\
    \hline
\end{tabular}
\caption{The sample misclassification rates on the brain cancer data set.}
\label{tab:brain}
\end{table}

The sample misclassification rates are reported in Table \ref{tab:brain}. In this case, we again see that our algorithm with all three different values of $\lambda$ achieve the smallest sample misclassification rate of 0.22, whereas other four biclustering methods all have sample misclassification rates around 0.34. This result indicates that even on larger gene expression data sets with more than two biclusters, our algorithm is still able to significantly outperform other biclustering methods such as KM, PL, SBC, and HSC.

\subsection{Prostate Cancer Gene Expression Data Set}
The third data set consists of 92 samples and 1288 genes. There are four types of samples: 27 samples correspond to benign prostate tissues, and 13, 32, 20 samples correspond to prostate cancer tissues of three different stages, respectively. 

Again, we try to use the elbow method to select the appropriate number of biclusters $k$ for our algorithm. In Figure \ref{fig:prostate}, we plot the losses for $k$ from 1 to 10 when applying our algorithm with $\lambda = 0$ to the prostate cancer data set. In this case, it is clear that our algorithm should select $k=4$ as the number of biclusters, which is also the true number of row clusters.

\begin{figure}[ht]
\centering
\includegraphics[width=0.6\textwidth, height=0.4\textwidth]{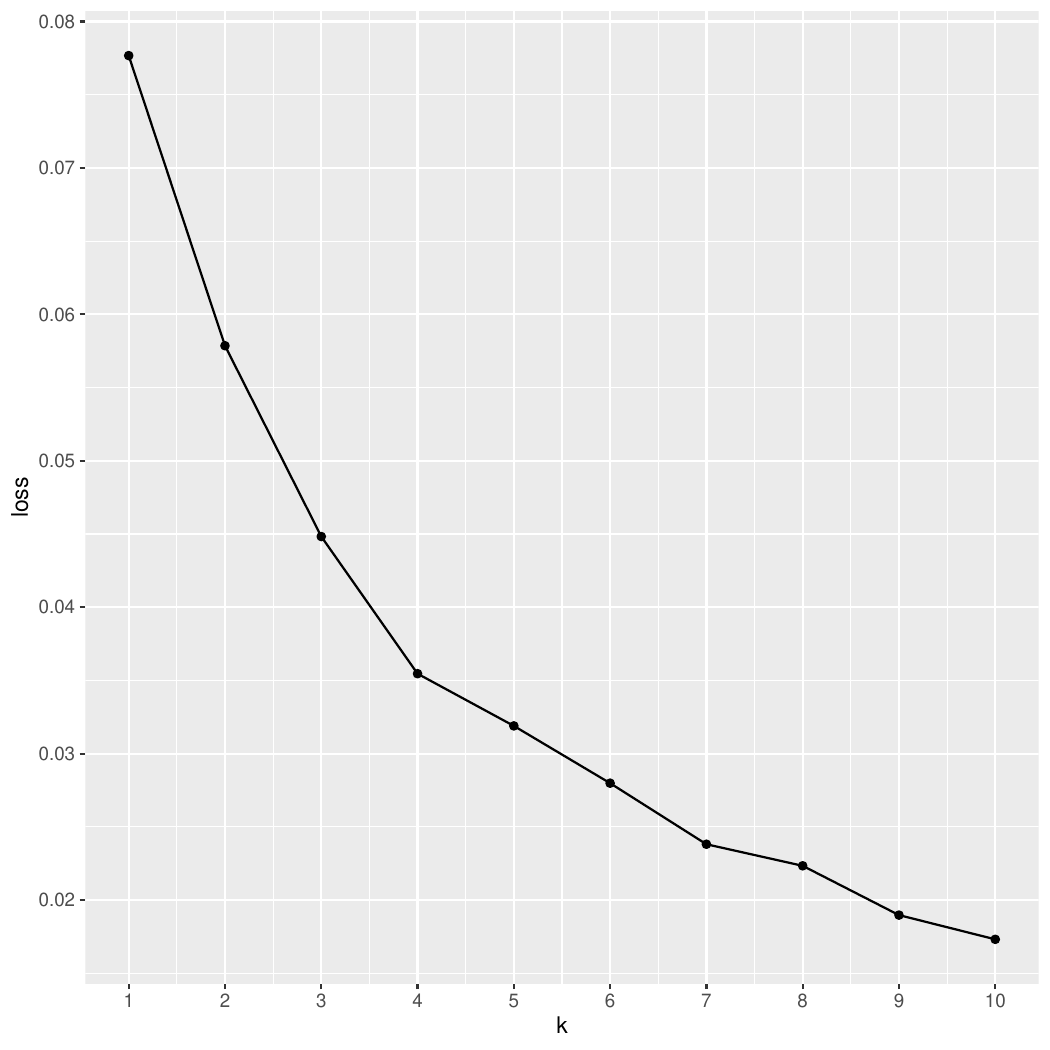}
\caption{Different losses for $k$ from 1 to 10 on the prostate cancer data set.}
\label{fig:prostate}
\end{figure}

Having selected $k=4$ as the number of biclusters, we again apply our algorithm with three different $\lambda$ values: 0, 0.1 and 1. We also apply KM, PL, SBC, and HSC, with the number of row clusters set to 4. All other settings of the biclustering algorithms are the same as in the first and second application.

\begin{table}[ht]
\centering
\begin{tabular}{c c c c c c c}
    \hline
    AKM ($\lambda = 0$) & AKM ($\lambda = 0.1$) & AKM ($\lambda = 1$) & KM & PL & SBC & HSC \\
    \hline 
    0.5217 & 0.4239 & 0.4239 & 0.5652 & 0.5109 & 0.5543 & 0.5652 \\
    \hline
\end{tabular}
\caption{The sample misclassification rates on the prostate cancer data set.}
\label{tab:prostate}
\end{table}

The sample misclassification rates are reported in Table \ref{tab:prostate}. In this case, we see that AKM with $\lambda = 0.1$ and $\lambda = 1$ achieve the smallest sample misclassification rate of 0.4239. AKM with $\lambda = 0$ has a sample misclassification rate of 0.5217, which is slightly worse than PL but better than SBC, KM, and HSC. This result once again demonstrates our algorithm's ability to achieve better performance at clustering samples on gene expression data sets compared to other biclustering algorithms such as KM, PL, SBC, and HSC.

\section{Discussion}\label{sec:discussion}
In this paper, we have provided a new formulation of the biclustering problem based on the idea of minimizing the empirical clustering risk. We have developed and proved a consistency result with respect to the empirical clustering risk. Since the optimization problem is combinatorial in nature, finding the global minimum is computationally infeasible. In light of this fact, we have proposed a simple and novel algorithm that finds a local minimum by alternating the use of $k$-means clustering between columns and rows, and released an R package \texttt{akmbiclust} on CRAN that implements the algorithm. We have also provided a probabilistic interpretation of the optimization problem, and proposed extending our method by adding penalization terms. We have evaluated and compared the performance of our algorithm to other related biclustering methods on both simulated data and real-world gene expression data sets. The results have demonstrated that our algorithm is able to detect meaningful structures in the data and outperform other competing biclustering methods in a lot of situations.

One big advantage of our algorithm is its simplicity: the $k$ biclusters can be found simply by applying an adapted version of the $k$-means clustering algorithm between columns and rows alternately. However, the simplicity comes at the expanse of flexibility: by assigning every row and every column to one and only one bicluster, our method excludes the possibility of overlapping biclusters. Although allowing the biclusters to overlap might be a more reasonable assumption in some cases, we argue that trading off some flexibility for more simplicity is a worthwhile choice for many applications.

In the future, we plan to explore a more general setting of biclustering: biclustering on graphs. The idea is that each column of $\mathbf{X}$ represents a vertex in a graph $G$, and each row of $\mathbf{X}$ represents a measurement on the vertices. The graph structure of $G$ imposes some restrictions on the column partitions, namely the columns in every bicluster should correspond to vertices in a connected subgraph of $G$. The problem formulation in Section \ref{sec:theory} can be considered as the special case where there is no restriction on the column partitions, and that is equivalent to the graph $G$ being a complete graph with $m$ vertices.

Under this setting of biclustering on graphs, a consistency result could be developed and proved in a very similar way. The only difference is that the optimization would be over all possible choices of column partitions that ``preserve'' the graph structure of $G$. The main theorem still holds true after adding the requirement of $\mathbf{I} \in \mathcal{I}(G)$, where $\mathcal{I}(G)$ denote the set of column partitions that ``preserve'' the graph structure of $G$. However, the algorithm presented in Section \ref{sec:algorithm} no longer applies to this setting, and in general people need to either do exhaustive searches over all column partitions in the space of $\mathcal{I}(G)$, or find some heuristic method that could efficiently search through the space of $\mathcal{I}(G)$. This would depend on the specific graph structure of $G$.

\bibliography{ref}

\begin{thebibliography}{40}
\providecommand{\natexlab}[1]{#1}
\providecommand{\url}[1]{\texttt{#1}}
\expandafter\ifx\csname urlstyle\endcsname\relax
  \providecommand{\doi}[1]{doi: #1}\else
  \providecommand{\doi}{doi: \begingroup \urlstyle{rm}\Url}\fi

\bibitem[Banerjee et~al.(2007)Banerjee, Dhillon, Ghosh, Merugu, and
  Modha]{banerjee2007generalized}
A.~Banerjee, I.~Dhillon, J.~Ghosh, S.~Merugu, and D.~S. Modha.
\newblock A generalized maximum entropy approach to {B}regman co-clustering and
  matrix approximation.
\newblock \emph{Journal of Machine Learning Research}, 8:\penalty0 1919--1986,
  2007.

\bibitem[Ben-Dor et~al.(2002)Ben-Dor, Chor, Karp, and
  Yakhini]{ben2002discovering}
A.~Ben-Dor, B.~Chor, R.~Karp, and Z.~Yakhini.
\newblock Discovering local structure in gene expression data: the
  order-preserving submatrix problem.
\newblock In \emph{Proceedings of the Sixth Annual International Conference on
  Computational Biology}, pages 49--57, 2002.

\bibitem[Bergmann et~al.(2003)Bergmann, Ihmels, and
  Barkai]{bergmann2003iterative}
S.~Bergmann, J.~Ihmels, and N.~Barkai.
\newblock Iterative signature algorithm for the analysis of large-scale gene
  expression data.
\newblock \emph{Physical Review E}, 67\penalty0 (3):\penalty0 031902, 2003.

\bibitem[Biau et~al.(2008)Biau, Devroye, and Lugosi]{biau2008performance}
G.~Biau, L.~Devroye, and G.~Lugosi.
\newblock On the performance of clustering in {H}ilbert spaces.
\newblock \emph{IEEE Transactions on Information Theory}, 54\penalty0
  (2):\penalty0 781--790, 2008.

\bibitem[Chen et~al.(2013)Chen, Sullivan, and Kosorok]{chen2013biclustering}
G.~Chen, P.~F. Sullivan, and M.~R. Kosorok.
\newblock Biclustering with heterogeneous variance.
\newblock \emph{Proceedings of the National Academy of Sciences}, 110\penalty0
  (30):\penalty0 12253--12258, 2013.

\bibitem[Cheng and Church(2000)]{cheng2000biclustering}
Y.~Cheng and G.~M. Church.
\newblock Biclustering of expression data.
\newblock In \emph{Proceedings of the Eighth International Conference on
  Intelligent Systems for Molecular Biology}, volume~8, page~93, 2000.

\bibitem[Chi et~al.(2017)Chi, Allen, and Baraniuk]{chi2017convex}
E.~C. Chi, G.~I. Allen, and R.~G. Baraniuk.
\newblock Convex biclustering.
\newblock \emph{Biometrics}, 73\penalty0 (1):\penalty0 10--19, 2017.

\bibitem[Chi et~al.(2020)Chi, Gaines, Sun, Zhou, and Yang]{chi2020provable}
E.~C. Chi, B.~J. Gaines, W.~W. Sun, H.~Zhou, and J.~Yang.
\newblock Provable convex co-clustering of tensors.
\newblock \emph{Journal of Machine Learning Research}, 21:\penalty0 1--58,
  2020.

\bibitem[Cho et~al.(2004)Cho, Dhillon, Guan, and Sra]{cho2004minimum}
H.~Cho, I.~S. Dhillon, Y.~Guan, and S.~Sra.
\newblock Minimum sum-squared residue co-clustering of gene expression data.
\newblock In \emph{Proceedings of the Fourth SIAM International Conference on
  Data Mining}, pages 114--125. SIAM, 2004.

\bibitem[de~Souto et~al.(2008)de~Souto, Costa, de~Araujo, Ludermir, and
  Schliep]{de2008clustering}
M.~C. de~Souto, I.~G. Costa, D.~S. de~Araujo, T.~B. Ludermir, and A.~Schliep.
\newblock Clustering cancer gene expression data: a comparative study.
\newblock \emph{BMC Bioinformatics}, 9\penalty0 (1):\penalty0 497, 2008.

\bibitem[Dhillon(2001)]{dhillon2001co}
I.~S. Dhillon.
\newblock Co-clustering documents and words using bipartite spectral graph
  partitioning.
\newblock In \emph{Proceedings of the Seventh ACM SIGKDD International
  Conference on Knowledge Discovery and Data Mining}, pages 269--274, 2001.

\bibitem[Dhillon et~al.(2003)Dhillon, Mallela, and
  Modha]{dhillon2003information}
I.~S. Dhillon, S.~Mallela, and D.~S. Modha.
\newblock Information-theoretic co-clustering.
\newblock In \emph{Proceedings of the Ninth ACM SIGKDD International Conference
  on Knowledge Discovery and Data Mining}, pages 89--98, 2003.

\bibitem[Flynn and Perry(2020)]{flynn2020profile}
C.~Flynn and P.~Perry.
\newblock Profile likelihood biclustering.
\newblock \emph{Electronic Journal of Statistics}, 14\penalty0 (1):\penalty0
  731--768, 2020.

\bibitem[Getz et~al.(2000)Getz, Levine, and Domany]{getz2000coupled}
G.~Getz, E.~Levine, and E.~Domany.
\newblock Coupled two-way clustering analysis of gene microarray data.
\newblock \emph{Proceedings of the National Academy of Sciences}, 97\penalty0
  (22):\penalty0 12079--12084, 2000.

\bibitem[Govaert and Nadif(2003)]{govaert2003clustering}
G.~Govaert and M.~Nadif.
\newblock Clustering with block mixture models.
\newblock \emph{Pattern Recognition}, 36\penalty0 (2):\penalty0 463--473, 2003.

\bibitem[Govaert and Nadif(2005)]{govaert2005algorithm}
G.~Govaert and M.~Nadif.
\newblock An em algorithm for the block mixture model.
\newblock \emph{IEEE Transactions on Pattern Analysis and Machine
  Intelligence}, 27\penalty0 (4):\penalty0 643--647, 2005.

\bibitem[Govaert and Nadif(2008)]{govaert2008block}
G.~Govaert and M.~Nadif.
\newblock Block clustering with bernoulli mixture models: Comparison of
  different approaches.
\newblock \emph{Computational Statistics \& Data Analysis}, 52\penalty0
  (6):\penalty0 3233--3245, 2008.

\bibitem[Gu and Liu(2008)]{gu2008bayesian}
J.~Gu and J.~S. Liu.
\newblock Bayesian biclustering of gene expression data.
\newblock \emph{BMC Genomics}, 9\penalty0 (S4), 2008.

\bibitem[Han et~al.(2020)Han, Luo, Wang, and Zhang]{han2020exact}
R.~Han, Y.~Luo, M.~Wang, and A.~R. Zhang.
\newblock Exact clustering in tensor block model: Statistical optimality and
  computational limit.
\newblock \emph{arXiv preprint arXiv:2012.09996}, 2020.

\bibitem[Hartigan(1972)]{hartigan1972direct}
J.~A. Hartigan.
\newblock Direct clustering of a data matrix.
\newblock \emph{Journal of the American Statistical Association}, 67\penalty0
  (337):\penalty0 123--129, 1972.

\bibitem[Hochreiter et~al.(2010)Hochreiter, Bodenhofer, Heusel, Mayr,
  Mitterecker, Kasim, Khamiakova, Van~Sanden, Lin, Talloen,
  et~al.]{hochreiter2010fabia}
S.~Hochreiter, U.~Bodenhofer, M.~Heusel, A.~Mayr, A.~Mitterecker, A.~Kasim,
  T.~Khamiakova, S.~Van~Sanden, D.~Lin, W.~Talloen, et~al.
\newblock {FABIA}: factor analysis for bicluster acquisition.
\newblock \emph{Bioinformatics}, 26\penalty0 (12):\penalty0 1520--1527, 2010.

\bibitem[Hofmann and Puzicha(1999)]{hofmann1999latent}
T.~Hofmann and J.~Puzicha.
\newblock Latent class models for collaborative filtering.
\newblock In \emph{Proceedings of the Sixteenth International Joint Conference
  on Artificial Intelligence}, volume~2, pages 688--693, 1999.

\bibitem[Kluger et~al.(2003)Kluger, Basri, Chang, and
  Gerstein]{kluger2003spectral}
Y.~Kluger, R.~Basri, J.~T. Chang, and M.~Gerstein.
\newblock Spectral biclustering of microarray data: coclustering genes and
  conditions.
\newblock \emph{Genome Research}, 13\penalty0 (4):\penalty0 703--716, 2003.

\bibitem[Lazzeroni and Owen(2002)]{lazzeroni2002plaid}
L.~Lazzeroni and A.~Owen.
\newblock Plaid models for gene expression data.
\newblock \emph{Statistica Sinica}, pages 61--86, 2002.

\bibitem[Lee et~al.(2010)Lee, Shen, Huang, and Marron]{lee2010biclustering}
M.~Lee, H.~Shen, J.~Z. Huang, and J.~S. Marron.
\newblock Biclustering via sparse singular value decomposition.
\newblock \emph{Biometrics}, 66\penalty0 (4):\penalty0 1087--1095, 2010.

\bibitem[Li et~al.(2009)Li, Ma, Tang, Paterson, and Xu]{li2009qubic}
G.~Li, Q.~Ma, H.~Tang, A.~H. Paterson, and Y.~Xu.
\newblock {QUBIC}: a qualitative biclustering algorithm for analyses of gene
  expression data.
\newblock \emph{Nucleic Acids Research}, 37\penalty0 (15):\penalty0 e101, 2009.

\bibitem[Linder(2002)]{linder2002learning}
T.~Linder.
\newblock Learning-theoretic methods in vector quantization.
\newblock In \emph{Principles of Nonparametric Learning}, pages 163--210.
  Springer, 2002.

\bibitem[Madeira and Oliveira(2004)]{madeira2004biclustering}
S.~C. Madeira and A.~L. Oliveira.
\newblock Biclustering algorithms for biological data analysis: a survey.
\newblock \emph{IEEE/ACM Transactions on Computational Biology and
  Bioinformatics}, 1\penalty0 (1):\penalty0 24--45, 2004.

\bibitem[Mankad and Michailidis(2014)]{mankad2014biclustering}
S.~Mankad and G.~Michailidis.
\newblock Biclustering three-dimensional data arrays with plaid models.
\newblock \emph{Journal of Computational and Graphical Statistics}, 23\penalty0
  (4):\penalty0 943--965, 2014.

\bibitem[Murali and Kasif(2003)]{murali2002extracting}
T.~Murali and S.~Kasif.
\newblock Extracting conserved gene expression motifs from gene expression
  data.
\newblock In \emph{Pacific Symposium on Biocomputing}, pages 77--88. World
  Scientific, 2003.

\bibitem[Preli{\'c} et~al.(2006)Preli{\'c}, Bleuler, Zimmermann, Wille,
  B{\"u}hlmann, Gruissem, Hennig, Thiele, and Zitzler]{prelic2006systematic}
A.~Preli{\'c}, S.~Bleuler, P.~Zimmermann, A.~Wille, P.~B{\"u}hlmann,
  W.~Gruissem, L.~Hennig, L.~Thiele, and E.~Zitzler.
\newblock A systematic comparison and evaluation of biclustering methods for
  gene expression data.
\newblock \emph{Bioinformatics}, 22\penalty0 (9):\penalty0 1122--1129, 2006.

\bibitem[Segal et~al.(2001)Segal, Taskar, Gasch, Friedman, and
  Koller]{segal2001rich}
E.~Segal, B.~Taskar, A.~Gasch, N.~Friedman, and D.~Koller.
\newblock Rich probabilistic models for gene expression.
\newblock \emph{Bioinformatics}, 17:\penalty0 243--252, 2001.

\bibitem[Shabalin et~al.(2009)Shabalin, Weigman, Perou, and
  Nobel]{shabalin2009finding}
A.~A. Shabalin, V.~J. Weigman, C.~M. Perou, and A.~B. Nobel.
\newblock Finding large average submatrices in high dimensional data.
\newblock \emph{The Annals of Applied Statistics}, 3\penalty0 (3):\penalty0
  985--1012, 2009.

\bibitem[Sill et~al.(2011)Sill, Kaiser, Benner, and
  Kopp-Schneider]{sill2011robust}
M.~Sill, S.~Kaiser, A.~Benner, and A.~Kopp-Schneider.
\newblock Robust biclustering by sparse singular value decomposition
  incorporating stability selection.
\newblock \emph{Bioinformatics}, 27\penalty0 (15):\penalty0 2089--2097, 2011.

\bibitem[Tan and Witten(2014)]{tan2014sparse}
K.~M. Tan and D.~M. Witten.
\newblock Sparse biclustering of transposable data.
\newblock \emph{Journal of Computational and Graphical Statistics}, 23\penalty0
  (4):\penalty0 985--1008, 2014.

\bibitem[Tanay et~al.(2002)Tanay, Sharan, and Shamir]{tanay2002discovering}
A.~Tanay, R.~Sharan, and R.~Shamir.
\newblock Discovering statistically significant biclusters in gene expression
  data.
\newblock \emph{Bioinformatics}, 18:\penalty0 136--144, 2002.

\bibitem[Tanay et~al.(2004)Tanay, Sharan, Kupiec, and
  Shamir]{tanay2004revealing}
A.~Tanay, R.~Sharan, M.~Kupiec, and R.~Shamir.
\newblock Revealing modularity and organization in the yeast molecular network
  by integrated analysis of highly heterogeneous genomewide data.
\newblock \emph{Proceedings of the National Academy of Sciences}, 101\penalty0
  (9):\penalty0 2981--2986, 2004.

\bibitem[Tanay et~al.(2005)Tanay, Sharan, and Shamir]{tanay2005biclustering}
A.~Tanay, R.~Sharan, and R.~Shamir.
\newblock Biclustering algorithms: a survey.
\newblock \emph{Handbook of Computational Molecular Biology}, 9\penalty0
  (1-20):\penalty0 122--124, 2005.

\bibitem[Tang et~al.(2001)Tang, Zhang, Zhang, and
  Ramanathan]{tang2001interrelated}
C.~Tang, L.~Zhang, A.~Zhang, and M.~Ramanathan.
\newblock Interrelated two-way clustering: an unsupervised approach for gene
  expression data analysis.
\newblock In \emph{Proceedings 2nd Annual IEEE International Symposium on
  Bioinformatics and Bioengineering}, pages 41--48. IEEE, 2001.

\bibitem[Wang et~al.(2002)Wang, Wang, Yang, and Yu]{wang2002clustering}
H.~Wang, W.~Wang, J.~Yang, and P.~S. Yu.
\newblock Clustering by pattern similarity in large data sets.
\newblock In \emph{Proceedings of the 2002 ACM SIGMOD International Conference
  on Management of Data}, pages 394--405, 2002.

\end{thebibliography}

\end{document}